\PassOptionsToPackage{table,dvipsnames}{xcolor}
\documentclass[11pt,a4paper,logo,onecolumn,copyright]{tmlrmel}

\usepackage[numbers,sort&compress,square]{natbib}
\usepackage{times}
\usepackage{latexsym}
\usepackage{graphicx}
\usepackage{mathtools}
\usepackage{amssymb}
\usepackage{amsmath}
\usepackage{amsthm}
\usepackage{amsfonts}
\usepackage{mathrsfs}
\usepackage{dsfont}
\usepackage{nicefrac}
\usepackage{tabularx}
\usepackage{multirow}
\usepackage{threeparttable}
\usepackage{booktabs}
\usepackage{makecell}
\usepackage{adjustbox}
\usepackage{arydshln}
\usepackage{caption}
\usepackage{subcaption}
\usepackage{enumitem}
\usepackage{inconsolata}
\usepackage{wrapfig}
\usepackage{xspace}
\usepackage{pifont}
\usepackage{listings}
\usepackage[ruled,vlined]{algorithm2e}

\tcbuselibrary{breakable,skins,listings}
\tcbset{
  enhanced,
  attach boxed title to top left={
    yshift=-3mm,
    yshifttext=-1mm,
    xshift=3mm
  },
  boxed title style={size=small}
}

\newtheorem{theorem}{Theorem}[section]
\newtheorem{proposition}[theorem]{Proposition}
\newtheorem{lemma}[theorem]{Lemma}
\newtheorem{corollary}[theorem]{Corollary}

\newtheorem{assumption}[theorem]{Assumption}
\newtheorem{remark}[theorem]{Remark}

\SetKwInput{KwInput}{Input}
\SetKwInput{KwOutput}{Output}
\SetKwInOut{KwRequire}{Require}
\SetKwInOut{KwEnsure}{Ensure}
\DontPrintSemicolon

\newcommand{\cmark}{\textcolor{ForestGreen}{\ding{51}}}
\newcommand{\xmark}{\textcolor{red}{\ding{55}}}
\newcommand{\model}{\textsc{ReDe}\xspace}

\definecolor{shallowgray}{HTML}{F5F5F5}

\newtcblisting{promptbox}[1][]{
  enhanced,
  colback=shallowgray,
  colframe=black,
  boxsep=2pt,
  left=8pt,
  right=8pt,
  top=6pt,
  bottom=6pt,
  title={\small #1},
  listing only,
  listing options={
    basicstyle=\ttfamily\footnotesize,
    breaklines=true,
    breakatwhitespace=false,
    columns=fullflexible,
    keepspaces=true,
    showstringspaces=false
  }
}

\newcommand{\answerTODO}[1][]{\textcolor{red}{\bfseries [TODO]}}
\newcommand{\justificationTODO}[1][]{\textcolor{red}{\bfseries [TODO]}}

\title{Reasoning Denoiser: Denoising Reasoning Traces for Hallucination Detection in Large Reasoning Models}

\author[1]{Junlin Fang}
\author[1]{Do Nguyen-Thanh}
\author[2]{Xiaogang Xu}
\author[3]{Zhen Fang}
\author[1,*]{Sean Du}

\affil[1]{College of Computing and Data Science, Nanyang Technological University}
\affil[2]{Zhejiang University}
\affil[3]{Australian Artificial Intelligence Institute, University of Technology Sydney}
\affil[*]{Corresponding author}

\correspondingauthors{
  \mbox{\texttt{junlin001@e.ntu.edu.sg}};
  \mbox{\texttt{xuefeng.du@ntu.edu.sg}}
}

\coderepo{July 24, 2026}

\begin{abstract}
Large reasoning models (LRMs) generate long reasoning traces before producing final answers. While these traces may contain useful signals for hallucination detection, harnessing them is non-trivial because long trajectories often include noisy steps that obscure the cues relevant to truthfulness assessment. In this paper, we identify two prevalent forms of reasoning noises, i.e., irrelevant steps and repetitive steps, and show that both substantially degrade hallucination detection performance. Existing confidence-based scores and naive embedding-based filtering fail to reliably separate noisy from informative steps. To address this challenge, we propose \model, a novel learning framework for denoising reasoning traces for hallucination detection. Specifically, \model leverages final-answer attention as an automatic supervision signal to shape the step-level representation space, yielding refined embeddings in which noisy steps can be reliably identified and filtered. \model can be readily plugged into diverse hallucination detectors by operating on the filtered reasoning trajectory after removing noisy steps. Extensive experiments on multiple reasoning benchmarks show that \model consistently improves detection performance over competitive baselines.
\end{abstract}

\begin{document}
\maketitle

\section{Introduction}
\label{sec:intro}

Hallucination detection for large reasoning models (LRMs) is critical for building reliable AI systems. Recent LRMs, such as OpenAI's o series~\citep{openai} and DeepSeek-R1~\citep{guo2025deepseek}, generate long reasoning traces before producing final answers. While these traces can improve reasoning capability, hallucination remains a serious reliability concern~\citep{luauditing,bao2025deepseek,guo2025deepseek,hout1}. Existing studies have begun to leverage reasoning traces for hallucination detection~\citep{sun2025detection,lu2026streaming}, motivated by the intuition that intermediate reasoning may reveal useful signals about the final answer correctness. However, harnessing reasoning traces is non-trivial. Recent evidence shows that, although long reasoning traces can improve reasoning performance, they can also obscure the cues useful for hallucination detection~\citep{cheng2025chain}. This suggests that naively incorporating the full reasoning trace does not consistently benefit detection.

\begin{figure*}[t]
    \centering
    \scalebox{1.0}{\includegraphics[width=\linewidth,trim=0 10 0 0,clip]{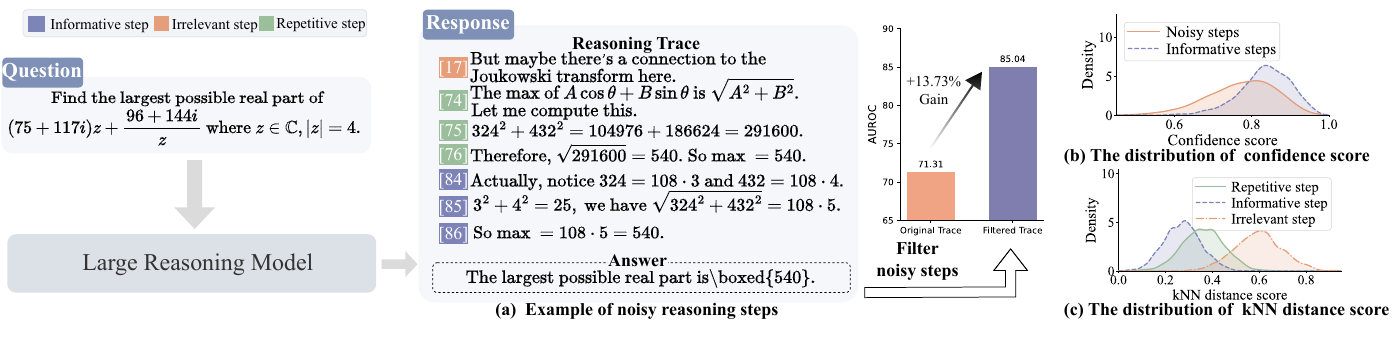}}
    \vspace{-0.6cm}
    \caption{\small 
(a) Long reasoning traces contain \textit{irrelevant} and \textit{repetitive} steps, which hurt hallucination detection; removing them yields clear improvements, as validated with supervised probing~\citep{azaria2023internal} using the last-layer embeddings.
(b) Confidence-based signals fail to distinguish informative and noisy steps, as their score distributions overlap heavily.
(c) Naive $k$NN-based on raw step embeddings can identify irrelevant steps but fails to separate repetitive from informative ones.
Experiments are based on AIME 2024~\citep{aime24} with Qwen3-8B~\citep{yang2025qwen3}. We annotate the reasoning steps into informative, irrelevant, and repetitive categories. Annotation details are in Appendix~\ref{sec:anno}.
}
    \vspace{-0.6cm}
    \label{fig:tease}
\end{figure*}

A key challenge is that long reasoning traces often contain \emph{noisy steps} that are weakly informative for hallucination detection. By examining LRM trajectories, we identify two prevalent forms of noisy steps: \emph{irrelevant steps}, which do not contribute meaningful problem-specific information, and \emph{repetitive steps}, whose information is already covered by later steps in the trajectory. As shown in Figure~\ref{fig:tease}(a), these noisy steps can substantially degrade hallucination detection performance, while removing them yields clear improvements. This observation suggests that \textit{identifying and filtering noisy reasoning steps is crucial} for effectively using reasoning traces in hallucination detection.

A natural approach is to identify noisy steps using confidence-based signals. Recent work has used confidence for improving reasoning-time generation~\citep{varshney2023stitch,SUGAR,CoCoLex,CCD}. However, these methods are designed to improve generation quality, rather than to identify which reasoning steps are informative for detecting hallucination in the final answer. Confidence reflects how certain the model is about what it generates, not how useful a step is for downstream hallucination detection. Indeed, Figure~\ref{fig:tease}(b) shows that the confidence distributions of informative and noisy steps overlap heavily, making confidence-based filtering ineffective.

Another possible route is to identify noisy steps in the embedding space. While irrelevant steps may behave as outliers and can sometimes be detected by simple representation-similarity methods such as $k$NN, this strategy remains inadequate for repetitive steps. Unlike irrelevant steps, repetitive steps are often semantically similar to informative ones and therefore remain close in the raw representation space. As shown in Figure~\ref{fig:tease}(c), naive $k$NN-based filtering can identify some irrelevant steps but fails to distinguish repetitive from informative ones. These limitations suggest that effective noisy-step filtering requires a signal beyond confidence or raw semantic similarity---one that better reflects each step's contribution to the final answer and to hallucination detection.

To address this challenge, we propose \textbf{Reasoning Denoiser (\model)}, a novel framework for filtering noisy reasoning steps by shaping the step-level representation space. Our key idea is to use the attention score between the final answer token and each reasoning step as a supervision signal, which requires no human annotation. Intuitively, this attention reflects how much a step contributes to forming the final answer: \emph{irrelevant steps} tend to receive low attention because they are weakly connected to the answer, while \emph{repetitive steps} can also receive low attention because their information has already been covered by later steps in the reasoning trajectory. Based on this observation, \model learns a lightweight projection that pulls informative steps closer together while pushing noisy steps away. The resulting refined step embeddings allow noisy steps to be reliably identified and filtered by simple distance-based methods such as $k$NN.

\model can be readily plugged into diverse hallucination detection methods by operating on the filtered reasoning trajectory after removing noisy steps, including probing-based~\citep{azaria2023internal,burnsdiscovering}, uncertainty-based~\citep{ren2023outofdistribution}, and verbalized methods~\citep{linteaching}. Extensive experiments on four representative reasoning tasks show that \model consistently improves hallucination detection performance across diverse datasets.  Specifically, on a representative benchmark TruthfulQA, \model improves detection performance by up to 18.69\% over using the original unfiltered reasoning trace, and achieves state-of-the-art performance of 87.32\%, demonstrating that reasoning-step denoising is an effective and general strategy for hallucination detection in LRMs.

Our key contributions are summarized as follows:

\begin{itemize}
\vspace{-0.6cm}
    \item We identify \textit{noisy reasoning steps} as an overlooked obstacle to hallucination detection. We further show that these noisy steps degrade hallucination detection performance, establishing reasoning-step filtering as an important solution for reliable hallucination detection.
    \item We propose Reasoning Denoiser (\model), a novel learning framework that leverages final-answer attention to shape the step-level representation space, enabling reliable filtering of noisy reasoning steps without human annotation.
    \item Extensive experiments (Section~\ref{sec:ablation} and Appendix~\ref{app:ablation}) and theoretical analyses (Appendix~\ref{app:theory}) show when and why \model improves detection performance. We also demonstrate that \model can be plugged into diverse hallucination detection methods for consistent gains. 
\end{itemize}

\section{Problem setup}
\label{sec:setup}

\vspace{-1em}
We study hallucination detection for large reasoning models that generate reasoning traces before final answers, which have become increasingly prominent in today’s foundation model landscape.

\noindent \textbf{Reasoning generation.}
Let an $L$-layer causal LRM take a prompt
$\mathbf{p}$ and generate a response autoregressively.
In reasoning-centric settings, we decompose the generated response into a reasoning trace
$\mathbf{C}=(\mathbf{c}_1,\ldots,\mathbf{c}_K)$ and a final answer $\mathbf{a}$, where each $\mathbf{c}_k$ is a contiguous token span corresponding to one reasoning step. Both the number of reasoning steps $K$ and the length of each step can vary across inputs.

\noindent \textbf{Hallucination detection.}
Given $(\mathbf{p}, \mathbf{C}, \mathbf{a})$, the goal is to determine whether the final answer $\mathbf{a}$ is truthful under the task-specific criterion. Let $y\in\{0,1\}$ denote the label, where $y=0$ means the answer is truthful and $y=1$ means it is hallucinated. We aim to learn a binary detector
$
    G(\mathbf{p},  \mathbf{C}, \mathbf{a})\in\{0,1\},
$
which predicts the truthfulness of the final answer~\citep{lu2026streaming,sun2025detection,wang2026joint,zhang2026harnessing}.

\noindent \textbf{Challenge of long reasoning traces.}
Unlike standard generation, the reasoning trace in an LRM is not uniformly useful for hallucination detection. Long traces often contain noisy steps that are weakly informative about the correctness of the final answer, thereby diluting the signals most relevant to truthfulness assessment~\citep{cheng2025chain}. Consequently, naively using the original reasoning trace may hurt detection performance. This motivates our goal of denoising the reasoning trace before hallucination detection. We provide analysis on the  noisy  reasoning steps next.

\vspace{-1em}

\section{Understanding noisy reasoning steps for hallucination detection}
\vspace{-0.5em}
\label{sec:pre}
Before introducing our method, we first seek to better understand the nature of noisy reasoning steps and to identify a suitable signal for detecting them. Our analysis leads to two main observations. First, the final-answer attention in LRMs provides a meaningful signal for distinguishing informative steps from noisy ones. Second, among semantically similar steps, earlier steps are more likely to be redundant, whereas later steps tend to preserve the information most relevant to the final answer. These observations motivate the design of our framework in Section~\ref{sec:method}.

\begin{figure}[t]
    \centering
    \includegraphics[width=1.0\linewidth, trim=0 60 0 0, clip]{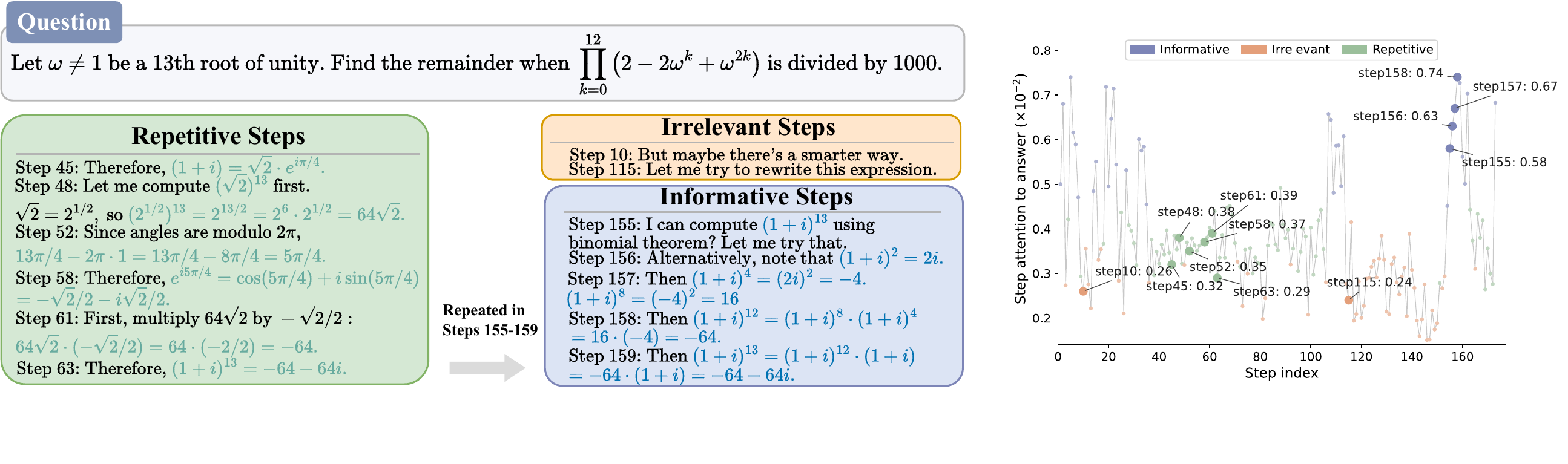}
    \vspace{-0.6cm}
    \caption{\small Attention scores from the final-answer token to each reasoning step on a representative AIME 2024~\citep{aime24} example with Qwen3-8B~\citep{yang2025qwen3}. Irrelevant and repetitive steps receive low attention, while informative steps receive high attention. The computation of attention scores is detailed in Section~\ref{sec:scoring}.}
    \label{fig:attention_case}
    \vspace{-0.6cm}
\end{figure}

\noindent \textbf{Final-answer attention reflects step informativeness.} We  examine whether the LRM's own internal signals can reveal which reasoning steps are useful for hallucination detection. Our key observation is that the attention score from the final-answer token to each reasoning step provides a meaningful relevance signal: steps that contribute more directly to the final answer tend to receive higher attention.

We verify this on our annotated AIME 2024~\citep{aime24} (\emph{cf.} Appendix~\ref{sec:anno}) using Qwen3-8B~\citep{yang2025qwen3}, where we compute the attention score assigned by the final-answer token to each reasoning step (formally defined in Eq.~\ref{eq:step_score}). Figure~\ref{fig:attention_case} shows a representative example. Three patterns emerge clearly: (1) \emph{irrelevant steps} receive low attention; (2) \emph{repetitive steps}, whose content is effectively subsumed by later steps, also receive low attention; and (3) \emph{informative steps} that directly support the final answer receive substantially higher attention. This pattern also holds at the dataset level. Averaged over all examples, informative steps receive a mean attention score of $6.5\times10^{-3}$, compared to $2.1\times10^{-3}$ for irrelevant steps and $3.3\times10^{-3}$ for repetitive steps. These results suggest that final-answer attention offers a practical and unsupervised signal for separating noisy steps from informative ones. We further validate the effectiveness of this signal in Appendix~\ref{app:attention_validation}.

\noindent \textbf{Earlier semantically similar steps are often redundant.}
Our analysis shows low-attention steps are not homogeneous: some are irrelevant, while others are part of repeated reasoning. We next focus on the latter case. When multiple semantically similar steps appear in a reasoning trajectory, an important question is which occurrence preserves the information most relevant to the answer. If later steps consistently retain the essential information better than earlier ones, this suggests earlier occurrences are redundant. To study this, we cluster reasoning steps based on embedding similarity and, from each cluster, 
\begin{wrapfigure}{r}{0.35\textwidth}
    \centering
    \vspace{-0.5em}
    \includegraphics[width=\linewidth]{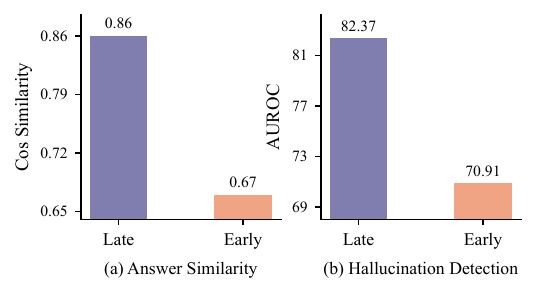}
    \vspace{-2em}
    \caption{\small (a) Cosine similarity between the answer embedding from the retained earliest/latest steps and that from the full reasoning trace. (b) Hallucination detection performance of linear probing based on the retained earliest/latest steps. Results on AIME 2024~\citep{aime24} with Qwen3-8B~\citep{yang2025qwen3}.}
    \label{fig:earlier_repetitive}
    \vspace{-1em}
\end{wrapfigure}
retain either the earliest or the latest step; details are in Appendix~\ref{app:clustering}.

For each retained subset, we concatenate the selected steps with the original question and answer, and then extract the resulting answer embedding from the LRM. We evaluate two criteria: (1) the cosine similarity between the new answer embedding and the original full-trace answer embedding, which measures how well the retained steps preserve the original reasoning signal; and (2) the hallucination detection performance obtained by applying a linear probe to the new answer embedding. Figure~\ref{fig:earlier_repetitive} shows retaining the latest steps consistently yields higher cosine similarity and stronger detection performance than retaining the earliest steps. This result suggests that \textit{earlier steps are often redundant}, whereas later steps might preserve the information  relevant for hallucination detection.

\vspace{-1em}

\section{Method}
\label{sec:method}
\begin{figure*}[t]
    \centering
    \includegraphics[width=\textwidth]{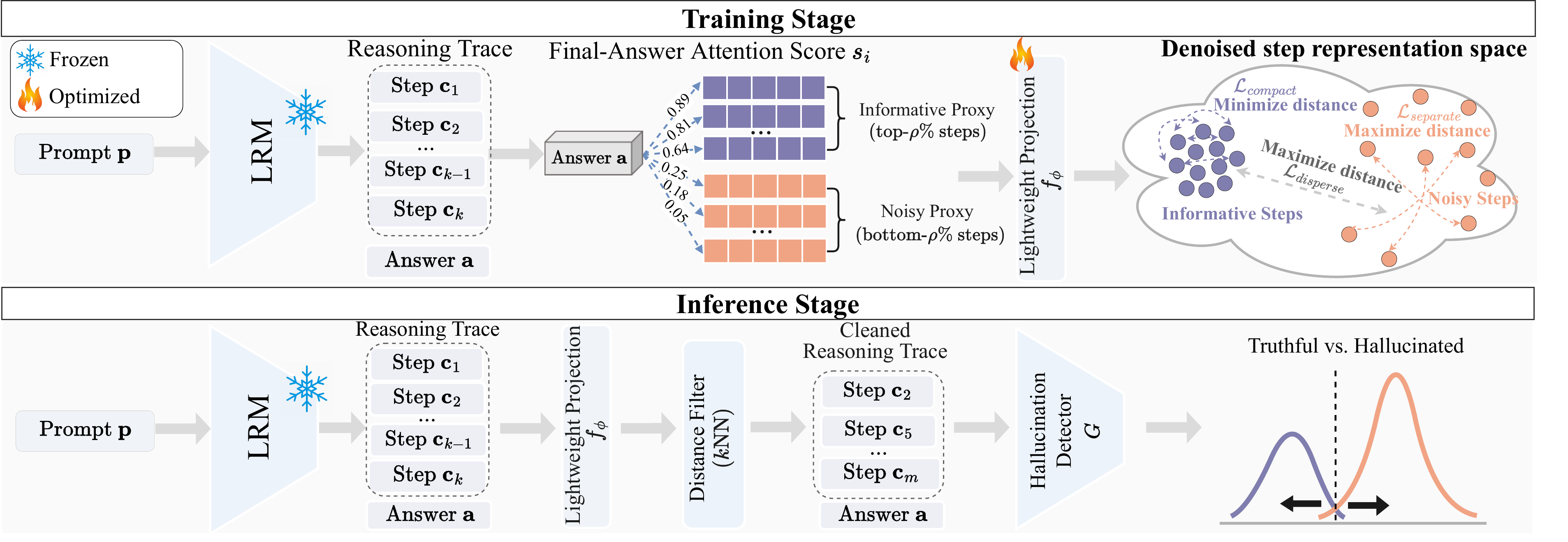}
    \vspace{-0.6cm}
    \caption{\small \textbf{Overview of the REDE framework for hallucination detection in LRMs.} During training, REDE scores each reasoning step via final-answer attention and learns a lightweight projection $f_\phi$ that maps step embeddings into a denoised representation space where informative and noisy steps are well separated. During inference, step embeddings are projected by $f_\phi$ and noisy steps are filtered via distance-based scoring for downstream hallucination detection.
    }
    \label{fig:overview}
    \vspace{-0.6cm}
\end{figure*}

Our novel framework \model addresses the challenge of denoising reasoning traces for hallucination detection. The key idea is to learn a step-level representation space in which informative reasoning steps form a compact region, while noisy steps become isolated and hence easy to filter. Rather than using the raw reasoning trace directly, \model first identifies an annotation-free supervision signal from the LRM itself, then learns a lightweight projection that reshapes the step embeddings, and finally performs filtering in the learned space before applying a downstream hallucination detector. Our framework is illustrated in Figure~\ref{fig:overview}.

\noindent \textbf{Framework overview.}
Given a prompt $\mathbf{p}$, a reasoning trace $\mathbf{C}=\{\mathbf{c}_1,\ldots,\mathbf{c}_K\}$, and a final answer $\mathbf{a}$ generated by an LRM, our framework encompasses three components addressing the following questions:
(1) \emph{How can we obtain a reliable supervision signal without human annotation?} We estimate how much each reasoning step contributes to the final answer using final-answer attention, yielding an annotation-free step score (Section~\ref{sec:scoring}).
(2) \emph{How can we learn a representation space favorable for noisy-step filtering?} Using the attention scores as supervision, we train a projection network that maps step embeddings into a new space where informative and noisy steps are more separable (Section~\ref{sec:representation_shaping}).
(3) \emph{How can we perform efficient filtering at test time?} We filter noisy steps in the shaped space and feed the resulting trace to a downstream hallucination detector (Section~\ref{sec:inference}).
The training and inference procedures are summarized in Algorithms~\ref{alg:training} and~\ref{alg:inference}.

\subsection{Scoring reasoning steps with final-answer attention}
\label{sec:scoring}

\model estimates the usefulness of each reasoning step for producing the final answer. As shown in Section~\ref{sec:pre}, the internal attention of LRMs provides a natural signal for this purpose.

Let $a_T$ denote the last token of answer $\mathbf{a}$, $\mathbf{x}_{1:N}=(x_1,\ldots,x_N)$ denote the full causal context visible to $a_T$, and $I_i \subseteq \{1,\ldots,N\}$ denote the set of token indices corresponding to the $i$-th reasoning step $\mathbf{c}_i$. We define the score of step $\mathbf{c}_i$ as the total attention mass assigned by $a_T$ to the tokens in that step:
\begin{equation}
\small
s_i
=
\sum_{t \in I_i}
\frac{
\exp\!\left(
\mathbf{q}(a_T)^\top \mathbf{k}(x_t)/\sqrt{\omega}
\right)
}{
\sum_{u=1}^{N}
\exp\!\left(
\mathbf{q}(a_T)^\top \mathbf{k}(x_u)/\sqrt{\omega}
\right)
},
\label{eq:step_score}
\end{equation}
where $\mathbf{q}(\cdot)$ and $\mathbf{k}(\cdot)$ are the query and key projections, and $\omega$ is the query/key projection dimension.

The score $s_i$ measures how strongly the final answer attends to step $\mathbf{c}_i$, which admits a simple interpretation.
If a step is \textit{irrelevant} to the answer, its key representations are weakly aligned with the query of the answer token, resulting in a low score.
If a step is \textit{repetitive}, its content is often subsumed by later steps, and the attention tends to concentrate on the later occurrence, again leading to a low score.
By contrast, informative steps that directly support the answer receive higher scores.
 
\noindent \textbf{From attention scores to representation shaping.}
While directly thresholding the attention score $s_i$ at test time and discarding low-attention steps is feasible, this approach has two important limitations. First, it requires recomputing attention scores for every test sample, which incurs non-trivial inference overhead. Second, although attention provides a meaningful signal when aggregated across many examples, per-instance attention scores can vary with reasoning trace length and structure, and prior work has shown that vanilla attention may not reliably transfer to downstream tasks without further learning~\citep{jain2019attention,wiegreffe2019attention}. We verify this empirically in Section~\ref{sec:ablation}, where direct attention-based filtering consistently underperforms our learning-based approach.

Instead of using attention as a direct filter, we use it as an annotation-free supervision signal to shape the step-level representation space. By learning a lightweight projection over many examples, \model converts the per-instance attention signal into a more stable embedding structure that reflects step informativeness. We introduce this process next.

\subsection{Learning a denoised step representation space}
\label{sec:representation_shaping}
 
We now describe how \model extracts step embeddings and learns a shaped representation space for noisy-step filtering.
 
\noindent \textbf{Step embedding extraction.}
A reasoning step often contains both informative and uninformative tokens.
To construct a step representation that emphasizes content-bearing tokens, we follow~\citep{miralles2025not} and adopt a perplexity-weighted aggregation of token embeddings.
Specifically, for the $i$-th reasoning step $\mathbf{c}_i$ spanning token indices $I_i$, we define
\begin{equation}
\small
\mathbf{e}_i
=
\frac{
\sum_{t \in I_i}
\mathrm{PPL}(x_t)\, \mathbf{h}(x_t)
}{
\sum_{t \in I_i}
\mathrm{PPL}(x_t)
},
\qquad
\mathrm{PPL}(x_t)
=
\frac{1}{p_\theta(x_t \mid x_{<t})},
\label{eq:ppl_embedding}
\end{equation}
where $\mathbf{h}(x_t)\in\mathbb{R}^d$ denotes the hidden representation of token $x_t$ and $p_\theta(x_t \mid x_{<t})$ is the next-token probability assigned by the LRM. This weighting assigns larger importance to less predictable tokens, which typically carry richer semantic information (different strategies are compared in Appendix~\ref{app:ablation_embedding}).

\noindent \textbf{Selecting informative and noisy proxies.}
For each training trace, we rank the steps by the attention scores $\{s_i\}_{i=1}^K$ (Eq.~\ref{eq:step_score}).
We denote by $\mathcal{T}$ the top-$\rho\%$ steps and by $\mathcal{B}$ the bottom-$\rho\%$ steps.
The set $\mathcal{T}$ serves as a proxy set of informative steps, whereas $\mathcal{B}$ serves as a proxy set of noisy steps (the effect of $\rho$ is studied in Appendix~\ref{app:ablation_ratio}).

\noindent \textbf{Projection learning.}
We train a lightweight projection network
$f_\phi:\mathbb{R}^d\rightarrow\mathbb{R}^{d'}$
to map each step embedding $\mathbf{e}_i$ to a projected embedding
$
\mathbf{z}_i = f_\phi(\mathbf{e}_i).
$
Throughout, the underlying LRM is frozen and only the projection parameters $\phi$ are optimized.

The goal of the projection is to induce an embedding distribution favorable for distance-based filtering: informative steps should form a compact cluster, noisy steps should remain scattered rather than collapse together, and the two groups should be well separated.
To realize this, we define the cosine similarity between projected steps $i$ and $j$ as
$
\mathrm{sim}_{ij}
=
\mathbf{z}_i^\top \mathbf{z}_j / (\|\mathbf{z}_i\|_2\|\mathbf{z}_j\|_2),
$
and optimize the following objective:
\begin{equation}
\small 
\mathcal{L}
=
\underbrace{
\mathbb{E}_{i\neq j,\; i,j\in\mathcal{T}}
~(1-\mathrm{sim}_{ij})
}_{\mathcal{L}_{\mathrm{compact}}}
+\;
\lambda_{\mathrm{disperse}} \cdot
\underbrace{
\mathbb{E}_{i\neq j,\; i,j\in\mathcal{B}}
~\mathrm{sim}_{ij}
}_{\mathcal{L}_{\mathrm{disperse}}}
+\;
\lambda_{\mathrm{separate}} \cdot 
\underbrace{
\mathbb{E}_{i\in\mathcal{T},\; j\in\mathcal{B}}
~\mathrm{sim}_{ij}
}_{\mathcal{L}_{\mathrm{separate}}}.
\label{eq:total_loss}
\end{equation}
\vspace{-1em}

Each term plays a distinct role.
$\mathcal{L}_{\mathrm{compact}}$ pulls informative steps together, forming a compact region.
$\mathcal{L}_{\mathrm{disperse}}$ discourages noisy steps from clustering with one another, so that they remain individually identifiable as outliers.
$\mathcal{L}_{\mathrm{separate}}$ pushes informative and noisy steps apart, improving their global separability.
Overall, these three terms shape a representation space in which noisy steps have large distances to the informative region and can therefore be filtered effectively using simple non-parametric rules (sensitivity to loss weights is analyzed in Appendix~\ref{app:ablation_weight}).
 
\noindent \textbf{Discussion.}
Our design is intentionally lightweight: the projection network operates on frozen step representations and requires only annotation-free supervision from the LRM itself.
Moreover, the learned embedding distribution is detector-agnostic---once the noisy steps are removed, the filtered trace can be passed to any downstream hallucination detector. We further provide a formal mathematical analysis showing that \textit{hallucination detection based on the filtered trace admits a bounded detection error}, where the bound depends directly on the quality of noisy-step filtering induced by the learned projection.
The full theorem and proof are deferred to Appendix~\ref{app:theory}.

\begin{table*}[t]
\centering
\caption{\small \textbf{Comparison with competitive hallucination detection methods.} 
``Single Generation'' indicates whether the method requires only one generation during inference. ``Supervision'' indicates whether the method requires ground-truth annotations during training or testing. We report the results of \model with CCS and supervised probing as downstream detectors. The best results are highlighted in \textbf{bold}.}
\vspace{-0.2cm}
\label{tab:main_baselines}
\small
\renewcommand{\arraystretch}{1.10}
\setlength{\tabcolsep}{0.5pt}
\resizebox{\textwidth}{!}{
\begin{tabular}{llcccccccccc}
\toprule
\multirow{3}{*}{\textbf{Method}} 
& \multirow{3}{*}{\makecell[l]{\textbf{Single}\\\textbf{Generation}}} 
& \multirow{3}{*}{\textbf{Supervision}} 
& \multicolumn{4}{c}{\textbf{Qwen3-8B}} 
& \multicolumn{4}{c}{\textbf{\makecell[c]{DeepSeek-R1-\\Distill-Llama-8B}}} \\
\cmidrule(lr){4-7} \cmidrule(lr){8-11}
& & 
& TruthfulQA & \textsc{Math} & CodeElo & \textsc{MultiHopQA}
& TruthfulQA & \textsc{Math} & CodeElo & \textsc{MultiHopQA} \\
\midrule

Lexical Similarity~\citep{lingenerating} 
& \xmark & \xmark 
& 52.00 & 49.02 & 55.44 & 58.26 
& 44.09 & 66.28 & 52.57 & 54.33 \\

SelfCheckGPT~\citep{manakul2023selfcheckgpt} 
& \xmark & \xmark 
& 53.11 & 47.20 & 51.85 & 52.61 
& 58.54 & 51.81 & 63.33 & 50.21 \\

Semantic Entropy~\citep{kuhnsemantic} 
& \xmark & \xmark 
& 54.35 & 40.69 & 53.76 & 66.83 
& 49.43 & 51.11 & 72.52 & 50.13 \\

EigenScore~\citep{cheninside} 
& \xmark & \xmark 
& 49.50 & 55.31 & 53.88 & 60.27 
& 45.31 & 51.17 & 55.20 & 58.12 \\

P(True)~\citep{kadavath2022language} 
& \cmark & \xmark 
& 64.47 & 75.16 & 55.62 & 75.59 
& 42.93 & 45.64 & 53.88 & 65.25 \\

HaloScope~\citep{du2024haloscope} 
& \cmark & \xmark 
& 67.22 & 61.43 & 57.19 & 61.03 
& 54.19 & 64.18 & 54.22 & 51.31 \\

TSV~\citep{park2025steer} 
& \cmark & \cmark 
& 69.91 & 71.88 & 59.42 & 67.43 
& 52.59 & 68.75 & 59.28 & 59.71 \\

CED~\citep{lee2024ced}
& \cmark & \xmark
& 70.29 & 68.42 & 60.23 & 72.18
& 54.51 & 65.30 & 61.72 & 57.32 \\

HaMI~\citep{niurobust}
& \cmark & \cmark
& 71.42 & 70.05 & 60.44 & 70.63
& 58.72 & 67.41 & 60.57 & 61.74 \\

\midrule
\multicolumn{11}{c}{\textit{LRM-based}} \\
\midrule

RHD~\citep{sun2025detection} 
& \cmark & \xmark 
& 58.08 & 64.15 & 62.21 & 58.16 
& 57.16 & 57.63 & 65.20 & 63.02 \\

RACE~\citep{wang2026joint} 
& \xmark & \xmark 
& 80.82 & 70.55 & 77.14 & 74.58 
& 61.23 & 62.62 & 59.02 & 64.62 \\

HalluGuard~\citep{zeng2026halluguard} 
& \cmark & \xmark 
& 51.67 & 53.58 & 56.52 & 55.73 
& 51.15 & 57.57 & 51.86 & 53.55 \\

ARS (CCS)~\citep{zhang2026harnessing} 
& \cmark & \xmark 
& 81.92 & 77.03 & 80.05 & 70.31 
& 75.42 & 72.18 & 73.97 & 68.40 \\

\rowcolor{gray!10}
\textbf{\model (CCS)} 
& \cmark & \xmark 
& \textbf{87.32}$^{\pm0.83}$ & 81.33$^{\pm2.14}$ & 82.17$^{\pm1.76}$ & 73.41$^{\pm2.53}$ 
& \textbf{78.30}$^{\pm1.91}$ & 74.24$^{\pm2.37}$ & 79.09$^{\pm0.62}$ & 72.85$^{\pm2.85}$ \\

\rowcolor{gray!10}
\textbf{\model (Probing)} 
& \cmark & \cmark 
& 86.44$^{\pm0.47}$ & \textbf{87.06}$^{\pm1.23}$ & \textbf{87.19}$^{\pm0.95}$ & \textbf{82.94}$^{\pm1.67}$ 
& 74.37$^{\pm2.08}$ & \textbf{83.91}$^{\pm0.71}$ & \textbf{82.55}$^{\pm1.42}$ & \textbf{78.71}$^{\pm1.89}$ \\

\bottomrule
\end{tabular}
}
\vspace{-1em}
\end{table*}

\vspace{-1em}
\subsection{Inference-time filtering and hallucination detection}
\label{sec:inference}

During inference, \model performs step filtering without requiring either attention scores or ground-truth labels. Given a test trace $\bar{\mathbf{C}}=\{\bar{\mathbf{c}}_1,\ldots,\bar{\mathbf{c}}_{\bar{K}}\}$, we first extract the step embeddings $\{\bar{\mathbf{e}}_i\}_{i=1}^{\bar{K}}$ using Eq.~\ref{eq:ppl_embedding} and project them into the learned space: $\bar{\mathbf{z}}_i = f_\phi(\bar{\mathbf{e}}_i)$.
We then compute a distance score $S_i$ for each step as its $k$-nearest neighbor cosine distance with respect to all other steps within the same trace: $S_i = 1 - \bar{\mathbf{z}}_i^{\top}\bar{\mathbf{z}}_i^{(k)} / (\|\bar{\mathbf{z}}_i\|_2\|\bar{\mathbf{z}}_i^{(k)}\|_2)$, where $\bar{\mathbf{z}}_{i}^{(k)}$ denotes the $k$-th nearest neighbor of $\bar{\mathbf{z}}_i$ among $\{\bar{\mathbf{z}}_j\}_{j \neq i}$.
Intuitively, informative steps cluster together in the shaped space and thus have small $k$-NN distances, while noisy steps are scattered and yield larger distances.
We remove the top-$\zeta\%$ steps with the largest distance scores and retain the remaining steps as the filtered trace $\bar{\mathbf{C}}^{\mathrm{fil}} \subseteq \bar{\mathbf{C}}$. The filtered trace is passed to a downstream hallucination detector $G$, which outputs
$
\hat{y}
=
G(\bar{\mathbf{p}}, \bar{\mathbf{C}}^{\mathrm{fil}}, \bar{\mathbf{a}}).
$
Importantly, \model is an agnostic denoising module for diverse downstream detection pipelines.

\vspace{-1em}
\section{Experiments}
\label{exp}

\textbf{Datasets and Models.}
We evaluate \model on four reasoning benchmarks: TruthfulQA~\citep{truthfulqa}, \textsc{MATH}, constructed by combining MATH500~\citep{math500}, AIME 2024~\citep{aime24}, and AIME 2025~\citep{aime25}, \textsc{CodeElo}~\citep{codeelo}, and \textsc{MULTIHOPQA}, which is built following prior work~\citep{sun2025detection} by sampling from HotpotQA~\citep{yang2018hotpotqa}, 2WikiMultihopQA~\citep{ho2020constructing}, MuSiQue~\citep{trivedi2022musique}, and Bamboogle~\citep{press2023measuring}. These datasets cover open-domain question answering, mathematical reasoning, code generation, and multi-hop question answering, respectively. For all datasets, we use 25\% of the available data for testing, reserve 100 examples for validation, and use the remaining data for training. We evaluate all methods using AUROC, and all reported results are averaged over three random seeds.

We evaluate our method on two widely used LRM families: Qwen3-8B/32B~\citep{yang2025qwen3} and DeepSeek-R1-Distill-Llama-8B/DeepSeek-R1-Distill-Qwen-32B~\citep{guo2025deepseek}.  Correctness labels are obtained using a
strong external judge model Qwen3-32B following prior work~\citep{yao2025reasoning} (robustness to alternative labeling methods is verified in Appendix~\ref{app:labeling}). By default, model outputs are generated with greedy decoding. Additional prompt, dataset, sampling strategies, and implementation details are provided in Appendices~\ref{app:prompts},~\ref{app:dataset_details},~\ref{app:sampling}, and~\ref{app:impl_details}, respectively. We also report qualitative case studies, results on GSM8K~\citep{cobbe2021training}, the impact of filtering on task accuracy, and additional evaluation metrics in Appendices~\ref{app:qualitative},~\ref{app:gsm8k},~\ref{appendix:filtering_accuracy}, and~\ref{app:metrics}, respectively.

\noindent \textbf{Baselines.} We first compare detection performance using the original unfiltered reasoning trace and our filtered one under four representative hallucination detectors: probing-based detector (Supervised Probing~\citep{azaria2023internal} and Contrast-Consistent Search, CCS~\citep{burnsdiscovering}), uncertainty-based detector (Perplexity~\citep{ren2023outofdistribution}), and verbalized detector (Verbalized Certainty~\citep{linteaching}). We then compare our method with a comprehensive set of baselines, including: (1) embedding-based methods: Truthfulness Separator Vector (TSV)~\citep{park2025steer}, HaloScope~\citep{du2024haloscope}, CED~\citep{lee2024ced}, and HaMI~\citep{niurobust}; (2) consistency-based methods: Lexical Similarity~\citep{lingenerating},  SelfCheckGPT~\citep{manakul2023selfcheckgpt}, Semantic Entropy~\citep{kuhnsemantic} and EigenScore~\citep{cheninside}; and (3) verbalized methods: P(True)~\citep{kadavath2022language}. We further include methods specifically designed for LRMs, including Reasoning and Answer Consistency Evaluation (RACE)~\citep{wang2026joint}, Reasoning Hallucination Detection (RHD)~\citep{sun2025detection}, HalluGuard~\citep{zeng2026halluguard}, and Answer-agreement Representation Shaping (ARS)~\citep{zhang2026harnessing}. For fair comparison, all baselines are evaluated on the same test sets under the default settings reported in their original papers. Implementation details of all baselines are provided in Appendix~\ref{app:impl_details}. 

\begin{table*}[t]
\centering
\caption{\small \textbf{Comparison of detecting hallucination using the original reasoning trace vs.\ our filtered one.}  All
values are percentages (AUROC) on different  detection methods. The best results are highlighted in \textbf{bold}.}
\vspace{-0.2cm}
\label{tab:main_detection}
\small
\setlength{\tabcolsep}{6pt}
\renewcommand{\arraystretch}{1.12}
\resizebox{\textwidth}{!}{
\begin{tabular}{llcccccccc}
\toprule
\multirow{2}{*}{\textbf{Model}} 
& \multirow{2}{*}{\textbf{Dataset}} 
& \multicolumn{2}{c}{\textbf{CCS}~\citep{burnsdiscovering}}
& \multicolumn{2}{c}{\textbf{Supervised Probing}~\citep{azaria2023internal}}
& \multicolumn{2}{c}{\textbf{Perplexity}~\citep{ren2023outofdistribution}}
& \multicolumn{2}{c}{\textbf{Verbalized Certainty}~\citep{linteaching}} \\
\cmidrule(lr){3-4}
\cmidrule(lr){5-6}
\cmidrule(lr){7-8}
\cmidrule(lr){9-10}
& & Original & Filtered & Original & Filtered & Original & Filtered & Original & Filtered \\
\midrule
\multirow{4}{*}{Qwen3-8B}
& TruthfulQA & 68.63 & \textbf{87.32} & 80.42 & \textbf{86.44} & 58.17 & \textbf{63.46} & 50.37 & \textbf{61.10} \\
& \textsc{Math} & 63.27 & \textbf{81.33} & 81.05 & \textbf{87.06} & 65.16 & \textbf{71.03} & 57.19 & \textbf{65.13} \\
& CodeElo & 61.64 & \textbf{82.17} & 82.82 & \textbf{87.19} & 61.78 & \textbf{72.03} & 78.89 & \textbf{83.89} \\
& \textsc{MultiHopQA} & 58.19 & \textbf{73.41} & 77.84 & \textbf{82.94} & 52.79 & \textbf{65.84} & 63.75 & \textbf{69.50} \\
\midrule
\multirow{4}{*}{\makecell[l]{DeepSeek-R1-\\Distill-Llama-8B}}
& TruthfulQA & 54.14 & \textbf{78.30} & 70.02 & \textbf{74.37} & 52.79 & \textbf{67.72} & 56.61 & \textbf{61.94} \\
& \textsc{Math} & 63.62 & \textbf{74.24} & 77.19 & \textbf{83.91} & 65.85 & \textbf{69.14} & 69.71 & \textbf{76.13} \\
& CodeElo & 55.16 & \textbf{79.09} & 75.07 & \textbf{82.55} & 69.35 & \textbf{78.49} & 79.13 & \textbf{84.49} \\
& \textsc{MultiHopQA} & 60.77 & \textbf{72.85} & 70.05 & \textbf{78.71} & 55.13 & \textbf{64.45} & 50.87 & \textbf{59.56} \\
\bottomrule
\end{tabular}
}
\vspace{-0.5cm}
\end{table*}

\subsection{Main results}
\label{sec:main_result}

\textbf{Effect of reasoning step filtering.} Table~\ref{tab:main_detection} compares hallucination detection using the original unfiltered reasoning trace versus our filtered trace across four representative detectors. Across all datasets and two LRMs, our step selection consistently yields substantial improvements, indicating that the benefit stems from filtering noisy steps rather than from any particular downstream detector. For example, on Qwen3-8B over TruthfulQA, AUROC improves from 68.63 to 87.32 for CCS, and from 80.42 to 86.44 for supervised probing. Similar trends hold on DeepSeek-R1-Distill-Llama-8B, confirming that \model generalizes across different LRM families. 

\begin{wrapfigure}{r}{0.5\textwidth}
\vspace{-0em}
\centering
\includegraphics[width=0.5\textwidth, trim=0 5 0 0,clip]{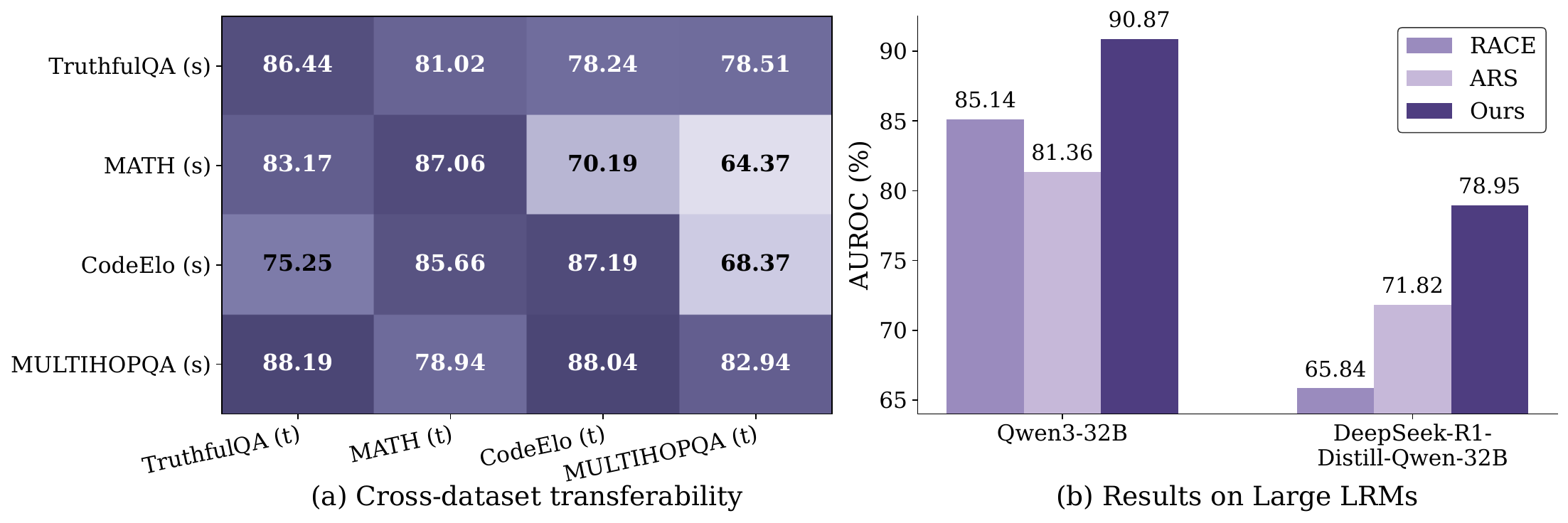}
\vspace{-0.5cm}
\caption{\small (a) \textbf{Cross-dataset transferability of \model.} Each row corresponds to a source dataset, and each column corresponds to a target dataset. (b) \textbf{Scalability to larger LRMs.} On Qwen3-32B and DeepSeek-R1-Distill-Qwen-32B, \model outperforms state-of-the-art baselines on TruthfulQA.}
\label{fig:generalizability}
\vspace{-1em}
\end{wrapfigure}
\noindent \textbf{Comparison with competitive baselines.} Table~\ref{tab:main_baselines} compares \model against a comprehensive set of baselines spanning embedding-based, consistency-based, verbalized, and LRM-specific detectors. 
\model achieves the best performance across nearly all settings. With supervised probing, \model reaches 87.06\% on \textsc{Math}, 87.19\% on CodeElo, and 82.94\% on \textsc{MultiHopQA} with Qwen3-8B. Even without supervision, \model with CCS substantially outperforms all unsupervised baselines, achieving 87.32\% on TruthfulQA. Notably, \model surpasses the strongest and most relevant representation-shaping baseline ARS~\citep{zhang2026harnessing} by a large margin (\emph{e.g.}, 81.33 vs.\ 77.03 on \textsc{Math}), demonstrating that reasoning-step denoising is a key strategy for hallucination detection. Computational cost is analyzed in Appendix~\ref{app:compute_time}.

\noindent \textbf{Generalization across data distributions.} Following prior work~\citep{park2025steer}, we train \model on a source dataset $s$ and apply the learned projection $f_\phi$ to a different target dataset $t$ for downstream detection. As shown in Figure~\ref{fig:generalizability} (results are based on Qwen3-8B with supervised probing as the downstream detector), \model transfers reliably across diverse datasets: for example, training on \textsc{MultiHopQA} and testing on CodeElo achieves 88.04\% AUROC, slightly surpassing the in-domain result of 87.19\%. 
These results indicate that the learned projection captures a general denoising signal that is largely invariant to dataset-specific surface form.

\noindent \textbf{Scalability to larger LRMs.} To assess scalability, we evaluate \model on Qwen3-32B and DeepSeek-R1-Distill-Qwen-32B. As shown in Figure~\ref{fig:generalizability}(b) (results are based on TruthfulQA with supervised probing as the downstream detector), \model consistently outperforms the two strongest baselines across both models: on Qwen3-32B, \model achieves 90.87\% AUROC, improving over RACE by 5.73\%; on 
DeepSeek-R1-Distill-Qwen-32B, \model reaches 78.95\%, outperforming RACE by 13.11\%. The consistent gains confirm that reasoning-step denoising remains effective in larger LRMs. Additional results over other benchmarks are provided in Appendix~\ref{app:large_lrm_more_results}.

\vspace{-1em}
\subsection{Analysis}
\label{sec:ablation}

We provide analysis to understand \model. Due to space limitations, we defer additional experiments to the Appendix, including (1) further ablations on design choices (Appendix~\ref{app:ablation}), (2) qualitative case studies (Appendix~\ref{app:qualitative}), (3) step selection accuracy against human annotations (Appendix~\ref{app:step_acc}), and (4) further analyses on the black-box applicability, and robustness across sampling, labeling, and additional metrics (Appendices~\ref{app:sampling}--\ref{app:metrics}).

\noindent \textbf{Ablation on selection strategies.}
We compare \model with six alternative strategies in Table~\ref{tab:selection_strategy}: random selection, four existing selection methods~\citep{liang2026hidden,cui2025stepwise,zeng2025pruning,li2026making}, and an LLM-as-Judge baseline using Qwen3-32B~\citep{yang2025qwen3}. Random selection degrades all detectors below the original baseline, confirming that naive subsampling discards useful signal. The other methods also fail to yield consistent improvements over the original reasoning trace, and in some cases even degrade detection performance. 
\begin{wraptable}{r}{0.45\textwidth}
\centering
\vspace{-0.5em}
\caption{\small \textbf{Comparing step selection strategies.} Results are on TruthfulQA using Qwen3-8B.}
\label{tab:selection_strategy}
\small
\setlength{\tabcolsep}{5pt}
\renewcommand{\arraystretch}{1.08}
\resizebox{\linewidth}{!}{%
\begin{tabular}{lcccc}
\toprule
\textbf{Selection Strategy}
& \textbf{CCS}~\citep{burnsdiscovering}
& \textbf{Probing}~\citep{azaria2023internal}
& \textbf{PPL}~\citep{ren2023outofdistribution}
& \textbf{Verb.}~\citep{linteaching} \\
\midrule
Original                         & 68.63 & 80.42 & 58.17 & 50.37 \\
\midrule
Random                           & 63.18 & 74.57 & 54.23 & 47.12 \\
STEP~\citep{liang2026hidden}      & 80.17 & 84.92 & 60.52 & 58.33 \\
SPIRIT~\citep{cui2025stepwise}    & 83.22 & 78.14 & 57.93 & 56.78 \\
ASAP~\citep{zeng2025pruning}      & 76.39 & 79.18 & 54.33 & 52.57 \\
Step Entropy~\citep{li2026making} & 82.33 & \textbf{87.31} & 61.49 & 60.03 \\
LLM-as-Judge~\citep{yang2025qwen3} & 70.11 & 74.81 & 59.17 & 50.94 \\
\rowcolor{gray!10}
\textbf{\model}                  & \textbf{87.32} & 86.44 & \textbf{63.46} & \textbf{61.10} \\
\bottomrule
\end{tabular}}
\vspace{-0.5em}
\end{wraptable}
This is because they are designed to improve generation quality and efficiency rather than to filter the proper noisy steps for improved detection, whereas \model specifically shapes representations with final-answer attention signal for this purpose. The LLM-as-Judge baseline also struggles because the long reasoning trace impairs its ability to reliably filter noisy steps. \model avoids this by operating in the representation space where noisy steps are identified via the LRM's own internal attention signal.

\noindent \textbf{Can attention scores alone identify noisy steps?} 
As shown in Table~\ref{tab:attention_vs_rede}, direct attention-based filtering consistently underperforms \model across all detectors (e.g., 80.51 vs.\ 87.32 for CCS). 
\begin{wraptable}{r}{0.4\textwidth}
\vspace{-0em}
\centering
\small
\setlength{\tabcolsep}{10pt}
\caption{\small Comparison of direct attention-based filtering and \model on TruthfulQA.}
\vspace{-0.2cm}
\resizebox{\linewidth}{!}{%
\begin{tabular}{lcc}
\toprule
\textbf{Detector} & \textbf{Attention} & \textbf{Ours} \\
\midrule
CCS                  & 80.51 & \textbf{87.32} \\
Supervised Probing   & 84.97 & \textbf{86.44} \\
Perplexity           & 60.11 & \textbf{63.46} \\
Verbalized Certainty & 59.73 & \textbf{61.10} \\
\bottomrule
\end{tabular}%
}
\label{tab:attention_vs_rede}
\vspace{-1.2em}
\end{wraptable}
Although attention scores provide a useful signal for distinguishing informative from noisy steps (e.g., 80.51 for filtering by final-answer attention signal vs. 58.08 for RHD~\citep{sun2025detection}), they might be sensitive to the heterogeneous reasoning traces and suboptimal as a direct filtering rule for detection. In contrast, \model uses attention as supervision to learn a projection that reshapes step embeddings, such that informative steps form a compact region while noisy steps become outliers. This enables more reliable distance-based filtering at test time, while avoiding attention computation during inference.

\noindent \textbf{Effect of individual loss components.}
We ablate each loss component by removing it from the full objective (Equation~\ref{eq:total_loss}).
As shown in Figure~\ref{fig:ablation}~(a) (results are based on TruthfulQA using Qwen3-8B with supervised probing), removing any single component degrades performance.
Removing $\mathcal{L}_\text{disperse}$ causes the largest drop (\emph{e.g.}, 86.44$\to$80.06), as noisy steps cluster together and become indistinguishable from informative ones with $k$NN.
Removing $\mathcal{L}_\text{compact}$ causes informative steps to spread out and overlap with noisy ones, and removing $\mathcal{L}_\text{separate}$ reduces the distance between the two groups.
We visualize the projected embeddings via t-SNE in Appendix~\ref{app:ablation_loss}, confirming that each loss induces a distinct embedding shaping effect.

\noindent \textbf{Effect of the drop ratio $\zeta$.} We vary the drop ratio $\zeta$ during inference, which controls the fraction of reasoning steps removed as noisy. As shown in Figure~\ref{fig:ablation}~(b), performance improves steadily from $\zeta=0$ (no filtering) and peaks at $\zeta=70\%$, then declines as $\zeta$ increases further. The results suggest that a large fraction of reasoning steps are noisy, while overly filtering will remove informative steps.

\noindent \textbf{Effect of different distance metrics.} We compare alternative distance metrics (cosine, Euclidean, and Mahalanobis) for identifying noisy steps in the projected space at test time. All three metrics yield comparable results, confirming the robustness of the learned projection. We defer the detailed comparison to Appendix~\ref{app:ablation_distance}.

\noindent \textbf{Effect of the output dimension $d'$.} We vary the output dimension $d'$ of the projection module $f_\phi$. As shown in Figure~\ref{fig:ablation}~(c), performance peaks at $d'=1024$ and remains competitive across a wide range, suggesting that the denoising signal can be effectively captured at various dimensionalities.

\noindent \textbf{How do different layers impact \model?} In Figure~\ref{fig:ablation}~(d), we train \model using embeddings from different layers. Consistent with prior findings~\citep{zhang2026harnessing, cheninside}, intermediate-to-late layers are most discriminative. \model improves separability across all layers (\emph{i.e.}, above 82.58 at all layers vs.\ 80.42 for the original trace), indicating that it does not rely on a specific layer depth.

\begin{figure}[t]
\centering
\includegraphics[width=\textwidth]{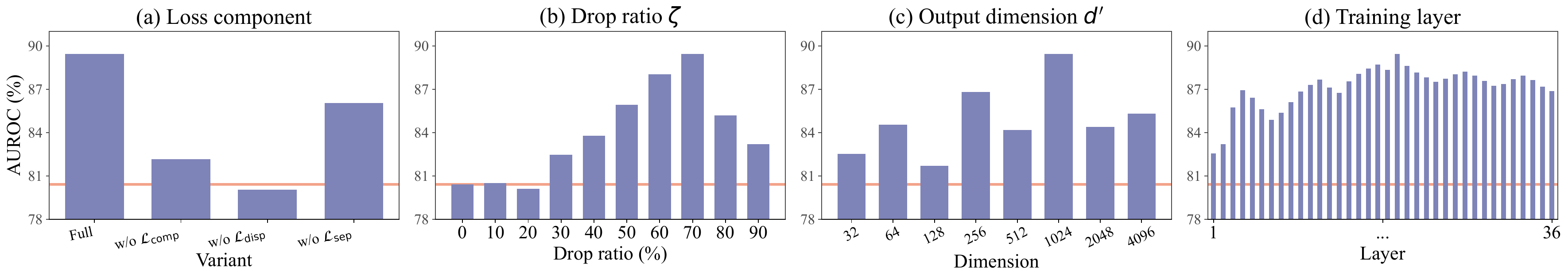}
\vspace{-0.6cm}
\caption{\small \textbf{Ablation studies.} (a) Effect of individual loss components, (b) effect of drop ratio $\zeta$, (c) effect of output dimension $d'$, and (d) effect of training layer. Results are based on TruthfulQA using Qwen3-8B with supervised probing. The orange horizontal line denotes the performance of using the unfiltered reasoning trace.}
\vspace{-0.7cm}
\label{fig:ablation}
\end{figure}

\vspace{-1em}

\section{Related work}

\textbf{Hallucination detection} has attracted broad interest across multiple paradigms, including logit- and probability-based uncertainty~\citep{kadavath2022language,ren2023outofdistribution,varshney2023stitch}, consistency-based scoring~\citep{burnsdiscovering,kuhnsemantic,lingenerating,manakul2023selfcheckgpt}, verbalized confidence~\citep{linteaching}, embedding-based approaches~\citep{azaria2023internal,cheninside,du2024haloscope,park2025steer,hu2025harp,bhatnagar2026drift,niurobust}, and methods that reveal mechanistic distinctions between factual and reasoning-driven hallucinations~\citep{luo2026two}. While effective for standard LLMs, these methods are not designed for LRMs, whose long reasoning traces introduce additional complexity. Empirical studies show that the reasoning trace actively obscures hallucination signals~\citep{yao2025reasoning,cheng2025chain}. Recent LRM-specific detectors exploit reasoning traces more directly~\citep{luauditing,lu2026streaming,sun2025detection,srey2025unsupervised,wang2026joint,zeng2026halluguard}, but none address the noise steps embedded within the trace itself. The most relevant work is ARS~\citep{zhang2026harnessing}, which shapes answer-level representations via latent perturbations at the trace boundary but treats the reasoning trace as a monolithic context, without considering the reasoning noise steps. \model fills this gap by filtering noisy steps before hallucination detection.

\noindent \textbf{Reasoning step selection} has been explored recently to select reasoning steps for better generation quality and efficiency~\citep{xia2025tokenskip,wu2025beyond,hou2025thinkprune,luo2025deconstructing,hong2025slim,tu2025deepprune,fu2025deep,do2025defines,hong2025pruning}. Representative approaches include measuring perplexity shifts upon step removal (SPIRIT~\citep{cui2025stepwise}), using first-token surprisal to identify logical pivots (ASAP~\citep{zeng2025pruning}), compressing traces via step-level entropy (Step Entropy~\citep{li2026making}), training hidden-state probes to evaluate step contributions (STEP~\citep{liang2026hidden}), and using various step-level importance signals to identify critical or redundant steps~\citep{wang2025stepwise,yuan2025not,an2025don}. All of them target generation rather than hallucination detection. As we show in Section~\ref{sec:ablation}, their selection signals do not reliably identify noisy steps for detection. More broadly, attention scores have been used for token-level filtering and KV cache compression in LLMs~\citep{zhang2023h2o,li2024snapkv,chen2024image}. Prior work applies attention patterns to identify pivotal reasoning sentences for interpretability~\citep{bogdan2025thought}. We explore using final-answer attention as a training and filtering signal to shape a representation space tailored for hallucination detection.

\section{Conclusion}

We propose \model, a novel framework for filtering noisy reasoning steps to detect hallucination in LRMs. Our framework leverages final-answer attention to shape the step-level representation space, enabling reliable identification and removal of noisy steps. \model is lightweight, requires no human annotation, and can be readily plugged into diverse downstream hallucination detectors. Extensive experiments show that \model consistently improves hallucination detection performance across datasets, models, and detector families. We hope our work can inspire future research on reasoning-trace denoising and more broadly on building reliable methods for assessing LRMs.

\newpage
\appendix
\section{Additional experimental details}
\label{app:setup}

\subsection{Input prompts}
\label{app:prompts}

We provide the detailed prompts used in our experiments for two purposes:
(1) generating the original reasoning traces and final answers, and
(2) evaluating the correctness of model-generated answers with an external judge model.

\noindent \textbf{Generation prompts.}
For each dataset, we prepend a task-specific system instruction before the question. The system instructions are as follows:
\begin{itemize}[leftmargin=*,itemsep=2pt]
    \item \textbf{TruthfulQA:} ``\textit{You are a factual question answering expert. Provide one concise and direct final answer.}''
    \item \textbf{MATH:} ``\textit{You are a mathematical reasoning expert. Solve the following problem step by step and give the final answer in the format \textbackslash boxed\{YOUR\_ANSWER\}. Answer concisely.}''
    \item \textbf{CodeElo:} ``\textit{You are a competitive programming expert. Solve the following Codeforces problem and provide a correct and efficient C++17 solution. Return only one final C++ code block. Answer concisely.}''
    \item \textbf{MULTIHOPQA:} ``\textit{You are a multi-hop question answering expert. Reason across the evidence\footnote{Here, ``evidence'' refers to the multiple pieces of supporting information that must be connected to answer a multi-hop question. For example, answering ``Was the director of \textit{Jaws} born in the same country as the author of \textit{Hamlet}?'' requires combining evidence about the film and its director, the play and its author, and their birth countries.} and provide one concise final answer.}''
\end{itemize}

These system instructions are combined with the model-specific prompt templates below.

\begin{figure}[htbp]
\centering
\begin{minipage}{0.96\linewidth}
\begin{promptbox}[Qwen Prompt]
<|im_start|>user
{system_instruction}
Question: {question}
<|im_end|>
<|im_start|>assistant
\end{promptbox}
\end{minipage}
\caption{\small Prompt used to generate reasoning traces and final answers for Qwen models.}
\label{fig:prompt_qwen}
\end{figure}

\begin{figure}[htbp]
\centering
\begin{minipage}{0.96\linewidth}
\begin{promptbox}[DeepSeek-R1 Prompt]
<|begin_of_sentence|><|User|>
{system_instruction}
Question: {question}
<|Assistant|><think>
\end{promptbox}
\end{minipage}
\caption{\small Prompt used to generate reasoning traces and final answers for DeepSeek-R1 models.}
\label{fig:prompt_deepseek}
\end{figure}

\noindent \textbf{Correctness evaluation prompt.}
Following prior work~\citep{zhang2026harnessing,yao2025reasoning,zheng2023judging}, we use Qwen3-32B as an external judge to assess whether a model-generated answer is correct, using the prompt shown in Figure~\ref{fig:prompt_judge}. Answers graded as \texttt{INCORRECT} or \texttt{NOT\_ATTEMPTED} are treated as hallucinations.

\begin{figure}[htbp]
\centering
\begin{minipage}{0.96\linewidth}
\begin{promptbox}[Judge Prompt]
Your job is to look at a question, multiple acceptable gold targets, and a predicted answer, and then assign a grade of either "CORRECT", "INCORRECT", or "NOT_ATTEMPTED".

IMPORTANT: The question has multiple acceptable correct answers provided as gold targets. The predicted answer is CORRECT if it matches any one of the gold targets according to the criteria below.

A predicted answer is CORRECT if:
- It fully contains the important information from at least one gold target.
- It does not contain any information that contradicts any gold target.
- Only semantic meaning matters; capitalization, punctuation, grammar, and order do not matter.
- Hedging or guessing is allowed as long as at least one gold target is fully included and the response contains no incorrect or contradictory statements.

A predicted answer is INCORRECT if:
- Any factual statement in the answer contradicts all the gold targets.
- Hedging statements are still considered incorrect if they introduce contradictory content.

A predicted answer is NOT_ATTEMPTED if:
- The important information from any gold target is not included.
- No statements in the answer contradict any of the gold targets.

Additional notes:
- Numerical answers must match the last significant figure of the gold target.
- If the gold target contains more information than the question requires, the predicted answer only needs to match the information requested by the question.
- Do not penalize omissions that are clearly inferable from the question.
- Do not penalize small typos in names if the intended meaning is clear.

Question: {question}
Gold targets: {targets}
Predicted answer: {predicted_answer}

Return only one of the following:
A
B
C

where:
A = CORRECT
B = INCORRECT
C = NOT_ATTEMPTED
\end{promptbox}
\end{minipage}
\caption{\small Prompt used to evaluate the correctness of model-generated answers with Qwen3-32B.}
\label{fig:prompt_judge}
\end{figure}

\begin{figure}[htbp]
\centering
\begin{minipage}{0.96\linewidth}
\begin{promptbox}[Verbalized Uncertainty Prompt for Qwen Models]
<|im_start|>user
A language model was asked: {question}
The model's answer was: {answer}

a. Very Certain
b. Fairly Certain
c. Moderately Certain
d. Somewhat Certain
e. Not Certain
f. Very Uncertain

Analyze its answer. How certain are you of the model's answer?
<|im_end|>
<|im_start|>assistant
\end{promptbox}
\end{minipage}
\caption{\small Prompt used for verbalized uncertainty estimation with Qwen models.}
\label{fig:prompt_verbalized_qwen}
\end{figure}

\begin{figure}[htbp]
\centering
\begin{minipage}{0.96\linewidth}
\begin{promptbox}[Verbalized Uncertainty Prompt for DeepSeek-R1 Models]
<|begin_of_sentence|><|User|>
A language model was asked: {question}
The model's answer was: {answer}

a. Very Certain
b. Fairly Certain
c. Moderately Certain
d. Somewhat Certain
e. Not Certain
f. Very Uncertain

Analyze its answer. How certain are you of the model's answer?
<|Assistant|><think>
</think>
Answer:
\end{promptbox}
\end{minipage}
\caption{\small Prompt used for verbalized uncertainty estimation with DeepSeek-R1 models.}
\label{fig:prompt_verbalized_deepseek}
\end{figure}

\subsection{Dataset details}
\label{app:dataset_details}
We evaluate \model on four tasks: TruthfulQA, \textsc{Math}, CodeElo, and \textsc{MultiHopQA}.

\noindent \textbf{TruthfulQA.} We use TruthfulQA~\citep{truthfulqa}, which contains 817 questions covering factual knowledge across diverse topics.

\noindent \textbf{\textsc{Math}.} We construct the evaluation set by combining MATH500~\citep{math500}, AIME 2024~\citep{aime24}, and AIME 2025~\citep{aime25}, resulting in 560 problems in total.

\noindent \textbf{CodeElo.} We use CodeElo~\citep{codeelo}, which contains 408 Codeforces problems requiring competition-level code generation.

\noindent \textbf{\textsc{MultiHopQA}.} We construct the evaluation set following prior work~\citep{sun2025detection} by sampling 1,000 questions from HotpotQA~\citep{yang2018hotpotqa}, 2WikiMultihopQA~\citep{ho2020constructing}, MuSiQue~\citep{trivedi2022musique}, and Bamboogle~\citep{press2023measuring}.

For all datasets, we reserve 25\% of the available examples for testing, use 100 examples for validation, and use the remaining examples for training. Unless otherwise specified, model outputs are generated with greedy decoding.

\noindent \textbf{Reasoning trace statistics.} Table~\ref{tab:trace_stats} reports the average number of reasoning tokens and reasoning steps per sample across datasets and models. Reasoning trace length varies substantially across tasks: CodeElo produces the longest traces (up to 7,808 tokens and 254.3 steps on average), while TruthfulQA and \textsc{MultiHopQA} are considerably shorter. Math problems fall in between, reflecting the iterative nature of symbolic reasoning. These statistics motivate the need for effective step filtering, as longer traces contain more opportunities for noisy steps to accumulate.

\begin{table}[h]
\centering
\caption{\small Average reasoning trace length (tokens) and number of reasoning steps per sample.}
\label{tab:trace_stats}
\small
\begin{tabular}{llrr}
\toprule
\textbf{Model} & \textbf{Dataset} & \textbf{Avg.\ Tokens} & \textbf{Avg.\ Steps} \\
\midrule
\multirow{4}{*}{Qwen3-8B}
  & TruthfulQA     & 2{,}294 & 65.8  \\
  & \textsc{Math}       & 5{,}606 & 145.0 \\
  & CodeElo    & 7{,}808 & 254.3 \\
  & \textsc{MultiHopQA}& 2{,}969 & 86.6  \\
\midrule
\multirow{4}{*}{DeepSeek-R1-Distill-Llama-8B}
  & TruthfulQA     & 2{,}090 & 54.5  \\
  & \textsc{Math}       & 4{,}251 & 112.4 \\
  & CodeElo    & 7{,}656 & 267.2 \\
  & \textsc{MultiHopQA}& 2{,}210 & 50.4  \\
\bottomrule
\end{tabular}
\end{table}

\subsection{Annotation details on AIME 2024}
\label{sec:anno}
To better understand the structure of noisy reasoning traces, we annotate the reasoning steps on AIME 2024 generated by Qwen3-8B. These annotations are used for analysis, including the examples in Figure~\ref{fig:tease} and the step-selection accuracy evaluation in Section~\ref{sec:ablation}.

\noindent \textbf{Step segmentation.}
Before annotation, we segment each reasoning trace into step-level units. Following prior work on reasoning hallucination analysis~\citep{sun2025detection}, we adopt a hybrid segmentation strategy that combines discourse-marker cues with formatting-based boundaries. Specifically, we split the trace at explicit discourse markers that typically indicate a transition in reasoning, such as ``Wait'', ``But'', ``However'', ``Hmm'', and ``Alternatively''. We further use formatting-based boundaries, such as paragraph breaks, when they correspond to coherent changes in reasoning. This procedure yields a sequence of contiguous reasoning steps for subsequent annotation.

\noindent \textbf{Annotation labels.}
Each reasoning step is assigned exactly one of three labels:
\begin{itemize}
    \item \textbf{Informative}, if it provides non-redundant, problem-specific information that directly supports the final answer.
    \item \textbf{Irrelevant}, if it does not contribute meaningful problem-specific information for solving the problem or verifying the answer.
    \item \textbf{Repetitive}, if its main content is semantically redundant with later reasoning steps and is expressed more completely or more decisively afterwards.
\end{itemize}

\noindent \textbf{Annotation procedure.} We adopt a two-stage procedure to ensure annotation quality. In the first stage, we conduct three successive rounds of labeling on the segmented reasoning steps, where the label guidelines are iteratively refined based on observed ambiguous cases before producing an initial label for each step. In the second stage, three graduate students with relevant expertise in mathematical reasoning and NLP \textit{independently} verify the initial labels on a randomly sampled subset without access to each other's judgments; disagreements with the initial label are resolved by majority vote among the verifiers, and the result serves as the final label for the sampled subset. The three verifiers achieve a Fleiss' $\kappa$ of $0.76$ on this subset, indicating substantial agreement~\citep{landis1977measurement}, which supports the reliability of the initial labels.

\noindent \textbf{Annotation statistics.} Table~\ref{tab:aime_annotation_stats} summarizes the distribution of the final labels across the three categories. Overall, repetitive steps form the largest category (2,127 steps, 48.8\%), while informative steps account for 1,178 steps (27.1\%) and irrelevant steps for 1,050 steps (24.1\%). Taken together, irrelevant and repetitive steps make up 3,177 steps (73.0\%) of the annotated traces, so we treat them as a single noisy class for the accuracy evaluation in Section~\ref{sec:ablation}.

\begin{table}[t]
\centering
\caption{\small Statistics of manual step annotations on AIME 2024 with Qwen3-8B.}
\label{tab:aime_annotation_stats}
\small
\setlength{\tabcolsep}{14pt}
\renewcommand{\arraystretch}{1.1}
\begin{tabular}{lc}
\toprule
\textbf{Step Category} & \textbf{Count} \\
\midrule
Informative & 1{,}178 \\
Irrelevant & 1{,}050 \\
Repetitive & 2{,}127 \\
\bottomrule
\end{tabular}
\end{table}

\subsection{Implementation details}
\label{app:impl_details}
\textbf{\model.}
For each reasoning trace, we first segment the trace into reasoning steps and extract a step-level embedding from the hidden representations of the backbone model. Unless otherwise specified, we use the final-layer hidden states and compute each step embedding by perplexity-weighted averaging over its tokens, as described in the main paper. The projection module $f_{\phi}$ is implemented as a two-layer MLP with ReLU activation. We train it using Adam with learning rate $1\mathrm{e}{-4}$, weight decay $1\mathrm{e}{-5}$, cosine learning rate decay, and batch size 128. The projection dimension, the loss weights, the top/bottom selection ratio, and the step drop ratio are selected on the validation set. At test time, noisy steps are identified using $k$NN distance in the projected space with $k=15$, and the remaining steps are used to construct the final detection representation. For the step score in Eq.~\ref{eq:step_score}, we uniformly average the per-head attention probabilities over all heads in the chosen layer.

\noindent \textbf{Downstream detectors.}
We evaluate \model with four downstream detectors: supervised probing~\citep{azaria2023internal}, CCS~\citep{burnsdiscovering}, verbalized uncertainty~\citep{linteaching}, and perplexity~\citep{ren2023outofdistribution,cheng2025chain}. For supervised probing, we follow prior work~\citep{azaria2023internal} and adopt an MLP classifier with a 512-dimensional hidden layer. We use dropout 0.6, Gaussian input noise with $\sigma=0.008$, and BatchNorm momentum 0.05. The classifier is trained with \texttt{BCEWithLogitsLoss} using SGD with learning rate $8\times10^{-3}$, momentum 0.9, and weight decay $5\times10^{-2}$ for 100 epochs with batch size 128. A ReduceLROnPlateau scheduler is used with factor 0.5, patience 7, and minimum learning rate $10^{-4}$. Unless otherwise specified, we extract representations from layers 0--35 and use the last-token representation for classification. For CCS, we follow its original setup~\citep{burnsdiscovering} and train a lightweight CCS classifier instantiated as a single linear layer, optimized with AdamW (learning rate 1e-3, weight decay 1e-2). The model is trained on balanced positive–negative embedding pairs constructed via difference vectors for 1000 epochs and with a logical consistency loss. We repeat training for 10 random initializations and retain the best-performing checkpoint. For verbalized uncertainty, we follow prior work~\citep{linteaching} and prompt the backbone model to self-assess its confidence using a six-level certainty scale and use the resulting score for AUROC evaluation. For perplexity-based detection, we follow prior work~\citep{ren2023outofdistribution,cheng2025chain} and use the mean perplexity over answer tokens as the detection score.

\noindent \textbf{Baselines.}
For lexical similarity~\citep{lingenerating} and semantic entropy~\citep{kuhnsemantic}, we reproduce them based on the original papers using multiple sampled responses from the same backbone model. SelfCheckGPT~\citep{manakul2023selfcheckgpt} is implemented with its NLI-based variant, which uses a fine-tuned DeBERTa-v3-large model\footnote{\url{https://huggingface.co/MoritzLaurer/DeBERTa-v3-large-mnli-fever-anli-ling-wanli}} to estimate the contradiction or entailment relation between the most-likely answer and sampled responses. For P(True)~\citep{kadavath2022language}, we reproduce it following the original paper. For TSV~\citep{park2025steer}, we follow the default settings described in the original paper and train it on the same dataset using embeddings extracted from the same layer as in our main experiments. For HaloScope~\citep{du2024haloscope}, we train it on the same unlabeled training set as our method and use pseudo-labels generated via PCA to train the lightweight MLP classifier. For EigenScore~\citep{cheninside}, we search the optimal number of sampled generations on the validation set and report the best AUROC. For RACE~\citep{wang2026joint}, we follow the original implementation and adopt the setting without pre-extracted chains of thought, where reasoning steps are summarized automatically by the built-in extractor. For RHD~\citep{sun2025detection}, we reproduce it based on the official code. For HalluGuard~\citep{zeng2026halluguard}, we follow the default configuration described in the original paper. For ARS~\citep{zhang2026harnessing}, we use the officially released codebase with the default hyperparameters reported in the paper. All baselines are evaluated on the same test sets under the default settings reported in their respective papers.

\subsection{Clustering details for Section~\ref{sec:pre}}
\label{app:clustering}

We describe the procedure for clustering semantically similar reasoning steps and retaining representative steps from each cluster. The full procedure is summarized in Algorithm~\ref{alg:clustering}.

Given a reasoning trace $\mathbf{C} = \{\mathbf{c}_1, \ldots, \mathbf{c}_K\}$, we first extract step embeddings $\{\mathbf{e}_1, \ldots, \mathbf{e}_K\}$ using perplexity-weighted averaging (Eq.~\ref{eq:ppl_embedding}). We then iteratively assign each step to an existing cluster if its cosine similarity to the cluster centroid exceeds a threshold $\tau_c$, or create a new cluster otherwise. After all steps are assigned, we retain either the earliest or latest step from each cluster based on its position in the trace. The retained steps are concatenated with the original question and answer, and the resulting sequence is fed through the LRM to extract the last-token hidden state of the answer as the answer embedding. We set $\tau_c = 0.7$.

\begin{algorithm}[h]
\caption{\small Step Clustering and Retention}
\label{alg:clustering}
\KwInput{Reasoning trace $\mathbf{C}=\{\mathbf{c}_1,\ldots,\mathbf{c}_K\}$, similarity threshold $\tau_c$, retention mode $m \in \{\texttt{earliest}, \texttt{latest}\}$}
\KwOutput{Retained reasoning trace $\mathbf{C}^{\mathrm{ret}}$}
Extract step embeddings $\{\mathbf{e}_1,\ldots,\mathbf{e}_K\}$ via Eq.~\ref{eq:ppl_embedding}\;
Initialize cluster list $\mathcal{G} \leftarrow \emptyset$\;
\For{$i=1,\ldots,K$}{
    $\mathrm{assigned} \leftarrow \mathrm{False}$\;
    
    \For{each cluster $G \in \mathcal{G}$}{
        Compute centroid $\bar{\mathbf{e}}_G = \frac{1}{|G|}\sum_{j \in G} \mathbf{e}_j$\;
        
        \If{$\cos(\mathbf{e}_i, \bar{\mathbf{e}}_G) \ge \tau_c$}{
            $G \leftarrow G \cup \{i\}$\;
            $\mathrm{assigned} \leftarrow \mathrm{True}$\;
            \textbf{break}\;
        }
    }
    
    \If{$\mathrm{assigned}=\mathrm{False}$}{
        Append singleton cluster $\{i\}$ to $\mathcal{G}$\;
    }
}
Initialize retained index set $\mathcal{I}^{\mathrm{ret}} \leftarrow \emptyset$\;
\For{each cluster $G \in \mathcal{G}$}{
    \eIf{$m=\texttt{earliest}$}{
        $\mathcal{I}^{\mathrm{ret}} \leftarrow \mathcal{I}^{\mathrm{ret}} \cup \{\min(G)\}$\;
    }{
        $\mathcal{I}^{\mathrm{ret}} \leftarrow \mathcal{I}^{\mathrm{ret}} \cup \{\max(G)\}$\;
    }
}
Sort the indices in $\mathcal{I}^{\mathrm{ret}}$ in ascending order\;
Construct $\mathbf{C}^{\mathrm{ret}}$ by retaining $\{\mathbf{c}_i \mid i \in \mathcal{I}^{\mathrm{ret}}\}$ in the original trace order\;
\KwRet{$\mathbf{C}^{\mathrm{ret}}$}\;
\end{algorithm}
\subsection{Training and inference algorithms of our \model}
\label{app:algorithms}
 
We provide the complete training and inference procedures of \model in Algorithm~\ref{alg:training} and Algorithm~\ref{alg:inference}, respectively. During training, \model computes step-level attention scores from the LRM's own attention mechanism and uses them as supervision to train a lightweight projection module $f_\phi$. During inference, \model applies the learned projection to identify and remove noisy steps without requiring attention scores or ground-truth labels, and passes the filtered reasoning trace to a downstream hallucination detector.

\begin{algorithm}[h]
\caption{\small Training Procedure of \model}
\label{alg:training}
\KwInput{Training set $\mathcal{D}_{\mathrm{train}}=\{(\mathbf{p}^{(i)}, \mathbf{C}^{(i)}, \mathbf{a}^{(i)})\}_{i=1}^{N}$, frozen LRM with parameters $\theta$, selection ratio $\rho$, loss weights $\lambda_{\mathrm{disperse}}$ and $\lambda_{\mathrm{separate}}$, number of epochs $E$}
\KwOutput{Trained projection module $f_\phi$}
Initialize projection module $f_\phi \colon \mathbb{R}^{d} \rightarrow \mathbb{R}^{d'}$\;
\tcp{Step 1: Construct proxy sets from final-answer attention}
\For{each training sample $(\mathbf{p}, \mathbf{C}, \mathbf{a})$ in $\mathcal{D}_{\mathrm{train}}$}{
    Let $\mathbf{C}=\{\mathbf{c}_1,\ldots,\mathbf{c}_K\}$\;
    
    \For{$i=1,\ldots,K$}{
        Compute PPL-weighted step embedding $\mathbf{e}_i$ via Eq.~\ref{eq:ppl_embedding}\;
        Compute step-level attention score $s_i$ via Eq.~\ref{eq:step_score}\;
    }
    
    $\mathcal{T} \leftarrow$ indices of the top $\rho\%$ steps ranked by $\{s_i\}_{i=1}^{K}$\;
    $\mathcal{B} \leftarrow$ indices of the bottom $\rho\%$ steps ranked by $\{s_i\}_{i=1}^{K}$\;
    
    Store $(\{\mathbf{e}_i\}_{i=1}^{K}, \mathcal{T}, \mathcal{B})$ for training\;
}
\tcp{Step 2: Train the projection module}
\For{epoch $=1,\ldots,E$}{
    \For{each mini-batch of stored $(\{\mathbf{e}_i\}, \mathcal{T}, \mathcal{B})$}{
        Compute projected embeddings $\mathbf{z}_i=f_\phi(\mathbf{e}_i)$ for all steps in the mini-batch\;
        
        Compute $\mathcal{L}_{\mathrm{compact}}$, $\mathcal{L}_{\mathrm{disperse}}$, and $\mathcal{L}_{\mathrm{separate}}$ from the projected embeddings and proxy sets\;
        
        Compute total loss
        $\mathcal{L}=\mathcal{L}_{\mathrm{compact}}+\lambda_{\mathrm{disperse}}\mathcal{L}_{\mathrm{disperse}}+\lambda_{\mathrm{separate}}\mathcal{L}_{\mathrm{separate}}$
        via Eq.~\ref{eq:total_loss}\;
        
        Update $\phi$ by gradient descent on $\mathcal{L}$\;
    }
}
\KwRet{$f_\phi$}\;
\end{algorithm}

\begin{algorithm}[h]
\caption{\small Inference Procedure of \model}
\label{alg:inference}
\KwInput{Test sample $(\bar{\mathbf{p}}, \bar{\mathbf{C}}, \bar{\mathbf{a}})$, trained projection $f_\phi$, number of nearest neighbors $k$, drop ratio $\zeta$, downstream detector $G$}
\KwOutput{Hallucination prediction $\hat{y} \in \{0,1\}$}
Let $\bar{\mathbf{C}}=\{\bar{\mathbf{c}}_1,\ldots,\bar{\mathbf{c}}_{\bar{K}}\}$\;
\tcp{Step 1: Extract and project step embeddings}
\For{$i=1,\ldots,\bar{K}$}{
    Compute PPL-weighted step embedding $\bar{\mathbf{e}}_i$ via Eq.~\ref{eq:ppl_embedding}\;
    Compute projected embedding $\bar{\mathbf{z}}_i=f_\phi(\bar{\mathbf{e}}_i)$\;
}
\tcp{Step 2: Compute within-trace $k$NN distance scores}
\For{$i=1,\ldots,\bar{K}$}{
    Let $\bar{\mathbf{z}}_{i}^{(k)}$ denote the $k$-th nearest neighbor of $\bar{\mathbf{z}}_i$ among $\{\bar{\mathbf{z}}_j\}_{j\ne i}$ within the same trace\;
    Compute distance score $S_i = 1 - \bar{\mathbf{z}}_i^{\top}\bar{\mathbf{z}}_i^{(k)} / (\|\bar{\mathbf{z}}_i\|_2\|\bar{\mathbf{z}}_i^{(k)}\|_2)$\;
}
\tcp{Step 3: Filter noisy steps}
$\mathcal{I}_{\mathrm{drop}} \leftarrow$ indices of the top-$\zeta\%$ steps with the largest scores in $\{S_i\}_{i=1}^{\bar{K}}$\;
Construct $\bar{\mathbf{C}}^{\mathrm{fil}}$ by retaining $\{\bar{\mathbf{c}}_i \mid i \notin \mathcal{I}_{\mathrm{drop}}\}$ in the original trace order\;
\tcp{Step 4: Downstream hallucination detection}
$\hat{y} \leftarrow G(\bar{\mathbf{p}}, \bar{\mathbf{C}}^{\mathrm{fil}}, \bar{\mathbf{a}})$\;
\KwRet{$\hat{y}$}\;
\end{algorithm}
\subsection{Compute resources and time}
\label{app:compute_time}

\textbf{Software and hardware.}
We conduct all experiments using Python 3.10 and PyTorch 2.3.1 on NVIDIA A100 GPUs with 80\,GB memory.

\noindent \textbf{Training and inference time.}
We report the training and inference time of \model on TruthfulQA with Qwen3-8B.
Based on tracked runs, training \model takes 681 seconds, and inference on the TruthfulQA test set takes 49 seconds.
For comparison, Semantic Entropy~\citep{kuhnsemantic} requires 1141 seconds, RACE~\citep{wang2026joint} requires 1507 seconds, and SelfCheckGPT~\citep{manakul2023selfcheckgpt} requires 2077 seconds under the same setting.
These results show that \model has a clear efficiency advantage over these baselines while achieving better hallucination detection performance.

\subsection{Licenses of existing assets}
\label{app:licenses}

We list the licenses of the datasets and pre-trained models used in this work. All assets are used in accordance with their respective licenses and intended research use.

\paragraph{Datasets.} TruthfulQA~\citep{truthfulqa}, 2WikiMultihopQA~\citep{ho2020constructing}, and CodeElo~\citep{codeelo} are released under Apache License 2.0; MATH~\citep{math500}, AIME 2024~\citep{aime24}, AIME 2025~\citep{aime25}, Bamboogle~\citep{press2023measuring}, and GSM8K~\citep{cobbe2021training} are released under MIT License; HotpotQA~\citep{yang2018hotpotqa} is released under CC BY-SA 4.0; and MuSiQue~\citep{trivedi2022musique} is released under CC BY 4.0.

\paragraph{Pre-trained models.} Qwen3-8B and Qwen3-32B~\citep{yang2025qwen3} are released under Apache License 2.0. DeepSeek-R1-Distill-Llama-8B and DeepSeek-R1-Distill-Qwen-32B~\citep{guo2025deepseek} are released under MIT License, with their base models additionally subject to the Llama 3.1 Community License and Apache License 2.0, respectively. The DeBERTa-v3-large-mnli-fever-anli-ling-wanli model used for the NLI variant of SelfCheckGPT~\citep{manakul2023selfcheckgpt} is released under MIT License.

\section{Additional ablation studies}
\label{app:ablation}
We conduct additional ablation studies to validate the key design choices of \model.
Unless otherwise specified, all experiments are conducted on TruthfulQA using Qwen3-8B with supervised probing as the downstream detector.

\subsection{Effect of training selection ratio}
\label{app:ablation_ratio}
During training, we select the top $\rho\%$ and bottom $\rho\%$ of steps by attention score to form the informative and noisy sets, respectively.
Table~\ref{tab:ablation_ratio} reports detection performance under varying $\rho$ across four downstream detectors.
A small $\rho$ provides insufficient contrast between the two groups, weakening the contrastive signal; a large $\rho$ reduces the separation between the two ends of the attention distribution, leading to less discriminative contrastive samples.
Performance consistently peaks in the range $\rho \in [20, 25]$, confirming that well-separated contrastive samples are important for learning a discriminative projection.
The training-time ratio $\rho$ complements the inference-time drop ratio $\zeta = 70\%$: $\rho$ selects the most contrastive samples at the two ends of the attention distribution for effective supervision, while $\zeta$ enables thorough denoising at test time.

\begin{table}[h]
\centering
\caption{\small Effect of training selection ratio $\rho$ (\%) on hallucination detection (AUROC, \%) on TruthfulQA with Qwen3-8B. The best result in each column is highlighted in \textbf{bold}.}
\label{tab:ablation_ratio}
\small
\resizebox{0.7\linewidth}{!}{%
\begin{tabular}{lcccc}
\toprule
$\rho$ (\%) 
& \textbf{CCS}~\citep{burnsdiscovering} 
& \textbf{Supervised Probing}~\citep{azaria2023internal} 
& \textbf{Perplexity}~\citep{ren2023outofdistribution} 
& \textbf{Verbalized Certainty}~\citep{linteaching} \\
\midrule
10 & 82.19 & 84.06 & 59.11 & 54.42 \\
15 & 84.52 & 82.29 & 61.05 & 57.82 \\
20 & \textbf{87.32} & 83.60 & \textbf{63.46} & \textbf{61.10} \\
25 & 87.05 & \textbf{86.44} & 62.74 & 59.18 \\
30 & 85.53 & 85.51 & 63.02 & 53.67 \\
35 & 84.41 & 86.07 & 60.31 & 57.93 \\
40 & 86.07 & 84.05 & 60.85 & 54.13 \\
\bottomrule
\end{tabular}%
}
\end{table}

\subsection{Visualization of loss ablation}
\label{app:ablation_loss}
Figure~\ref{fig:ablation_tsne} provides t-SNE visualizations corresponding to the loss ablation in Section~\ref{sec:ablation}. In the original embedding space, informative and noisy steps are heavily mixed. The full model produces a compact informative cluster clearly separated from scattered noisy steps, while removing any single loss leads to increased overlap.

\begin{figure}[h]
\centering
\includegraphics[width=\linewidth]{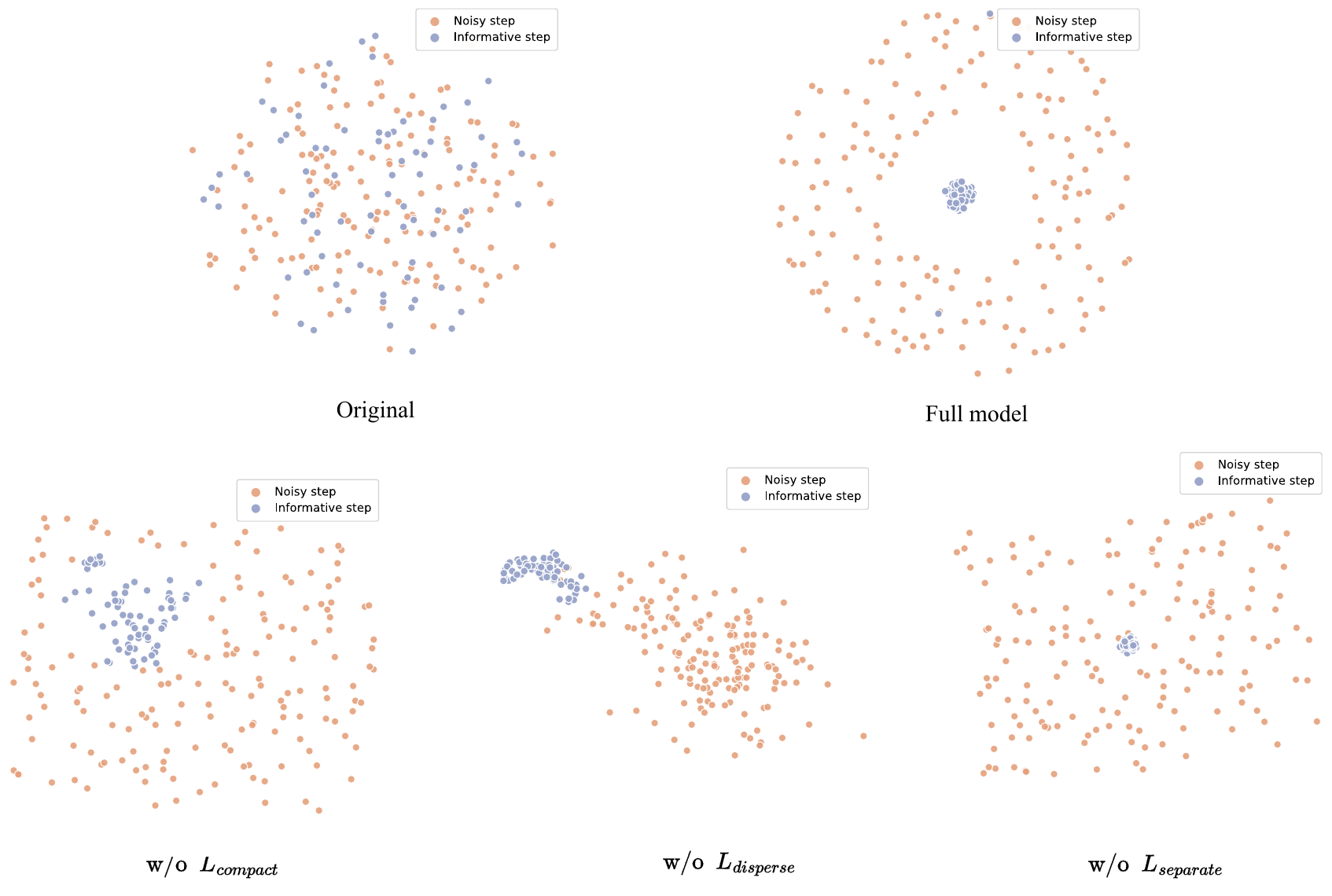}
\caption{\small t-SNE visualization of step embeddings under each ablated loss variant on AIME 2024~\citep{aime24} with Qwen3-8B. \textcolor{orange}{Orange} and \textcolor{blue!60}{blue} points denote noisy and informative steps, respectively.}
\label{fig:ablation_tsne}
\end{figure}

\subsection{Effect of different distance metrics}
\label{app:ablation_distance}
We compare alternative distance metrics for the $k$NN-based filtering in the projected space at test time. Our default method uses cosine distance. We additionally evaluate Mahalanobis distance and Euclidean distance. As shown in Table~\ref{tab:ablation_distance}, all three metrics yield strong and comparable results, confirming that the learned projection produces a well-separated score distribution where noisy steps can be reliably identified regardless of the specific distance metric.
\begin{table}[h]
\centering
\caption{\small Effect of distance metrics for $k$NN-based filtering on hallucination detection on TruthfulQA with Qwen3-8B using supervised probing.}
\label{tab:ablation_distance}
\small
\begin{tabular}{lc}
\toprule
\textbf{Distance Metric} & \textbf{Supervised Probing}~\citep{azaria2023internal} \\
\midrule
Cosine (Ours)   & \textbf{86.44} \\
Mahalanobis     & 81.51 \\
Euclidean       & 83.02 \\
\bottomrule
\end{tabular}
\end{table}

\subsection{Effect of loss weights}
\label{app:ablation_weight}
\model uses two weighting hyperparameters $\lambda_\text{disperse}$ and $\lambda_\text{separate}$ to balance the three loss components (Eq.~\ref{eq:total_loss}).
We independently vary each weight while fixing the other to its default value, and report detection performance with supervised probing on TruthfulQA with Qwen3-8B in Figure~\ref{fig:ablation_weight}.
Note that $\lambda = 0$ corresponds to removing the respective loss component entirely.
Performance remains stable across a moderate range of both weights, indicating that \model is not overly sensitive to the precise choice of loss weights.
A small weight weakens the corresponding geometric constraint, reducing the quality of the learned projection; a large weight dominates the objective and suppresses the contributions of the other two losses.

\begin{figure}[h]
\centering
\includegraphics[width=0.6\linewidth]{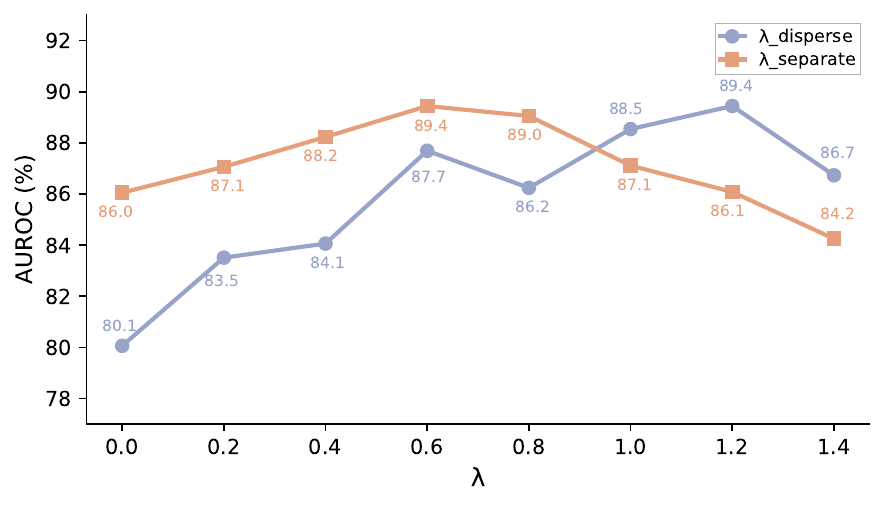}
\caption{\small Effect of loss weights $\lambda_\text{disperse}$ and $\lambda_\text{separate}$ on hallucination detection on TruthfulQA with Qwen3-8B, evaluated with supervised probing. $\lambda = 0$ indicates the corresponding loss is removed.}
\label{fig:ablation_weight}
\end{figure}

\subsection{Effect of the number of nearest neighbors $k$}
\label{app:ablation_k}
At inference time, we compute each step's distance score as its $k$-NN cosine distance with respect to all other steps within the same trace.
Table~\ref{tab:ablation_k} reports detection performance under varying $k$ across four downstream detectors.
Across all values of $k$, \model consistently outperforms the original unfiltered reasoning trace, demonstrating that the learned projection produces a stable density structure that is robust to the choice of $k$.
Performance peaks at $k = 15$ across all four detectors, confirming that this setting best captures the local density signal induced by the learned projection.
\begin{table}[h]
\centering
\caption{\small Effect of the number of nearest neighbors $k$ on hallucination detection (AUROC, \%) on TruthfulQA with Qwen3-8B. The best result in each column is highlighted in \textbf{bold}.}
\label{tab:ablation_k}
\small
\resizebox{0.7\linewidth}{!}{%
\begin{tabular}{lcccc}
\toprule
$k$
& \textbf{CCS}~\citep{burnsdiscovering}
& \textbf{Supervised Probing}~\citep{azaria2023internal}
& \textbf{Perplexity}~\citep{ren2023outofdistribution}
& \textbf{Verbalized Certainty}~\citep{linteaching} \\
\midrule
5  & 83.97 & 82.61 & 60.42 & 57.13 \\
10 & 85.34 & 85.78 & 61.89 & 60.27 \\
15 & \textbf{87.32} & \textbf{86.44} & \textbf{63.46} & \textbf{61.10} \\
20 & 86.18 & 85.39 & 62.07 & 59.45 \\
25 & 85.91 & 84.02 & 62.83 & 58.74 \\
30 & 84.05 & 84.71 & 60.51 & 58.06 \\
\bottomrule
\end{tabular}%
}
\end{table}

\subsection{Effect of step score aggregation}
\label{app:ablation_score}

In Eq.~\ref{eq:step_score}, we compute the step-level attention score $s_i$ by aggregating attention weights from the final answer token to each reasoning step. We ablate two design choices for this aggregation: the position of the answer token used as the query, and the layer from which the attention is extracted.

\noindent \textbf{Answer token position.} We compare four choices: the first token, the middle token (1/2), mean pooling over all answer tokens, and the last token. As shown in Table~\ref{tab:ablation_score}, using the last answer token yields the best performance across all downstream detectors, suggesting that it accumulates the most complete contextual information from the answer sequence and therefore provides the most informative relevance signal for each reasoning step.

\noindent \textbf{Layer.} Fixing the query to the last answer token, we further vary the transformer layer from which the attention score is computed. As shown in Figure~\ref{fig:attn_layer}, early layers (0--7) are generally less effective ($\sim 83$\% AUROC), while intermediate-to-late layers yield stronger results, with local peaks around layer 10, 20, and 22. Layers 27--28 show a mild drop to around 84.5\%, after which performance recovers and peaks at the final layer (86.44\%). This supports our default choice of computing the step score from the last answer token at the final transformer layer.

\begin{table}[h]
\centering
\caption{\small Effect of answer token position for step score computation on TruthfulQA with Qwen3-8B. The best result in each column is highlighted in \textbf{bold}.}
\label{tab:ablation_score}
\small
\resizebox{0.7\linewidth}{!}{%
\begin{tabular}{lcccc}
\toprule
\textbf{Token Position}
& \textbf{CCS}~\citep{burnsdiscovering}
& \textbf{Supervised Probing}~\citep{azaria2023internal}
& \textbf{Perplexity}~\citep{ren2023outofdistribution}
& \textbf{Verbalized Certainty}~\citep{linteaching} \\
\midrule
First token        & 86.17 & 82.15 & 54.09 & 52.96 \\
Middle token (1/2) & 85.09 & 85.35 & 57.48 & 55.72 \\
Mean pooling       & 86.71 & 86.34 & 61.30 & 60.27 \\
Last token (Ours)  & \textbf{87.32} & \textbf{86.44} & \textbf{63.46} & \textbf{61.10} \\
\bottomrule
\end{tabular}%
}
\end{table}

\begin{figure}[h]
\centering
\includegraphics[width=0.95\linewidth]{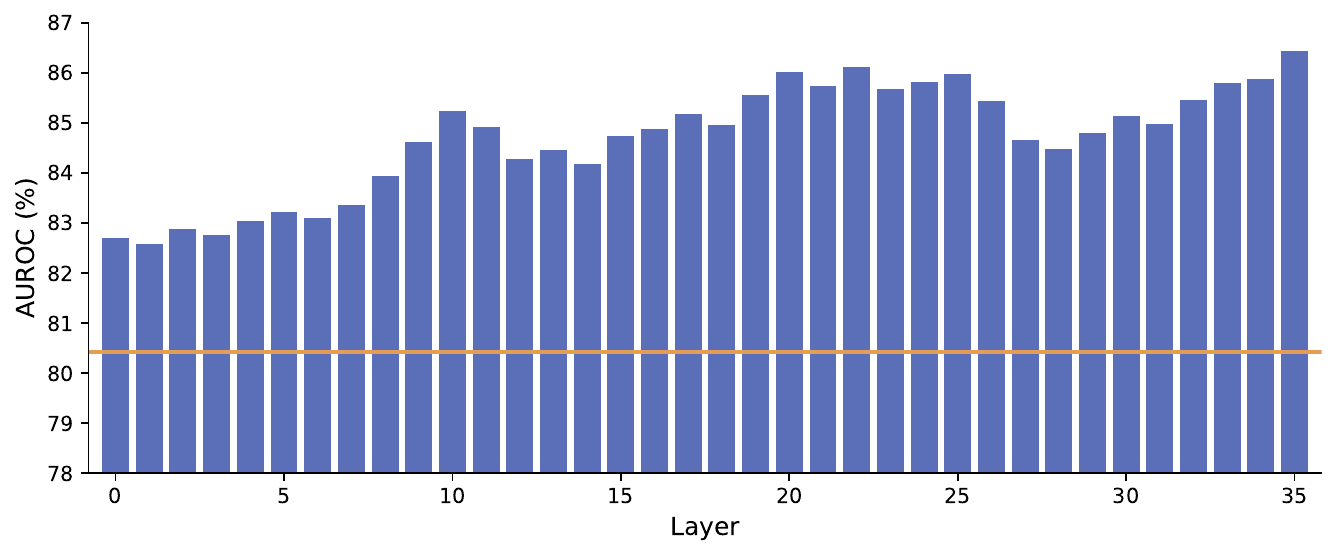}
\caption{\small Effect of the layer used to compute the attention score in Eq.~\ref{eq:step_score}. Results are on TruthfulQA with Qwen3-8B, evaluated with supervised probing. The orange horizontal line denotes the performance of using the unfiltered reasoning trace.}
\label{fig:attn_layer}
\end{figure}

\subsection{Effect of step embedding weighting}
\label{app:ablation_embedding}
In Eq.~\ref{eq:ppl_embedding}, we compute the step-level embedding $\mathbf{e}_k$ by applying perplexity (PPL)-weighted averaging over the token hidden states within each step.
Specifically, the token-level perplexity is defined as $\text{PPL}(x_{k_j}) = \exp(-\log p_\theta(x_{k_j} \mid x_{<k_j}))$, where $p_\theta(x_{k_j} \mid x_{<k_j})$ is the model's predicted probability for token $x_{k_j}$ given its preceding context.
Tokens with higher perplexity receive larger weights, as they are less predictable and therefore more likely to carry discriminative information for hallucination detection.
Table~\ref{tab:ablation_embedding} compares three step embedding strategies: using only the last token, uniform averaging over all tokens in the step, and PPL-weighted averaging.
The results show that PPL-weighting consistently achieves the best performance across all detectors, while the last-token representation is generally stronger than uniform averaging but remains inferior to PPL-weighting.

\begin{table}[h]
\centering
\caption{\small Effect of step embedding weighting strategy (AUROC, \%) on TruthfulQA with Qwen3-8B. The best result in each column is highlighted in \textbf{bold}.}
\label{tab:ablation_embedding}
\small
\resizebox{0.7\linewidth}{!}{%
\begin{tabular}{lcccc}
\toprule
\textbf{Weighting Strategy}
& \textbf{CCS}~\citep{burnsdiscovering}
& \textbf{Probing}~\citep{azaria2023internal}
& \textbf{Perplexity}~\citep{ren2023outofdistribution}
& \textbf{Verbalized}~\citep{linteaching} \\
\midrule
Uniform average     & 82.56 & 85.41 & 60.88 & 59.08 \\
Last token          & 86.18 & 86.05 & 63.31 & 60.09 \\
PPL-weighted (Ours) & \textbf{87.32} & \textbf{86.44} & \textbf{63.46} & \textbf{61.10} \\
\bottomrule
\end{tabular}%
}
\end{table}

\subsection{Effect of embedding extraction location}
\label{app:ablation_location}
Following prior work~\citep{du2024haloscope,park2025steer}, we investigate the effect of the multi-head attention (MHA) architecture on step-level representation quality.
Specifically, the MHA can be conceptually expressed as:
\begin{equation}
f_{i+1} = f_i + Q_i \operatorname{Attn}_i(f_i),
\end{equation}
where $f_i$ represents the output of the $i$-th transformer block, $\operatorname{Attn}_i(f_i)$ denotes the output of the self-attention module in the $i$-th block, and $Q_i$ is the output projection matrix of the attention sub-block.
We evaluate hallucination detection performance using step embeddings extracted from three different locations within the MHA architecture, as shown in Table~\ref{tab:ablation_location}.
The results show that the block output $f$ is a favorable choice for detecting hallucinations across both LRM architectures.

\begin{table}[h]
\centering
\caption{\small Effect of embedding extraction location on hallucination detection (AUROC, \%) based on supervised probing~\citep{azaria2023internal}. The best result in each column is highlighted in \textbf{bold}.}
\label{tab:ablation_location}
\small
\resizebox{0.7\linewidth}{!}{%
\begin{tabular}{lcccc}
\toprule
\multirow{2}{*}{\textbf{Embedding Location}} 
& \multicolumn{2}{c}{\textbf{Qwen3-8B}} 
& \multicolumn{2}{c}{\textbf{DeepSeek-R1-Distill-Llama-8B}} \\
\cmidrule(lr){2-3} \cmidrule(lr){4-5}
& \textsc{TruthfulQA} & \textsc{Math} & \textsc{TruthfulQA} & \textsc{Math} \\
\midrule
$\operatorname{Attn}(f)$          & 82.40 & 84.06 & 72.08 & 80.05 \\
$Q\operatorname{Attn}(f)$         & 85.50 & 83.44 & 73.18 & 78.52 \\
$f$ (Ours)                        & \textbf{86.44} & \textbf{87.06} & \textbf{74.37} & \textbf{83.91} \\
\bottomrule
\end{tabular}%
}
\end{table}

\section{Additional analysis}
\label{app:analysis}

\subsection{Qualitative analysis}
\label{app:qualitative}
We present step-level case studies and trace-level embedding visualizations to illustrate how \model improves hallucination detection in practice. All reasoning traces are generated by Qwen3-8B~\citep{yang2025qwen3} on the \textsc{Math} benchmark with greedy decoding, and detection scores are computed using supervised probing~\citep{azaria2023internal}.

\noindent \textbf{Step-level filtering examples.} Figures~\ref{fig:rede_sample1} and~\ref{fig:rede_sample2} show two representative traces after \model filtering, where each reasoning step is marked as either \textcolor{green!40!black}{kept} or \textcolor{black!30}{filtered}.
In Figure~\ref{fig:rede_sample1}, the model gives an incorrect answer due to an arithmetic error. The original 84-step trace contains substantial repetitive and tangential content, leading to a low detection score of 0.41. After \model removes 70\% of the steps, the score increases to 0.86, making the error more salient to the detector.
Figure~\ref{fig:rede_sample2} shows the complementary case, where the model reaches the correct answer through a verbose but valid derivation. The original trace receives a score of 0.77, incorrectly suggesting hallucination. After \model removes 71\% of the steps, the score drops to 0.28, reducing this false positive. Together, these examples show that \model improves detection for both hallucinated and truthful responses by filtering noisy reasoning steps.

\begin{figure}[h]
\centering
\includegraphics[width=\linewidth]{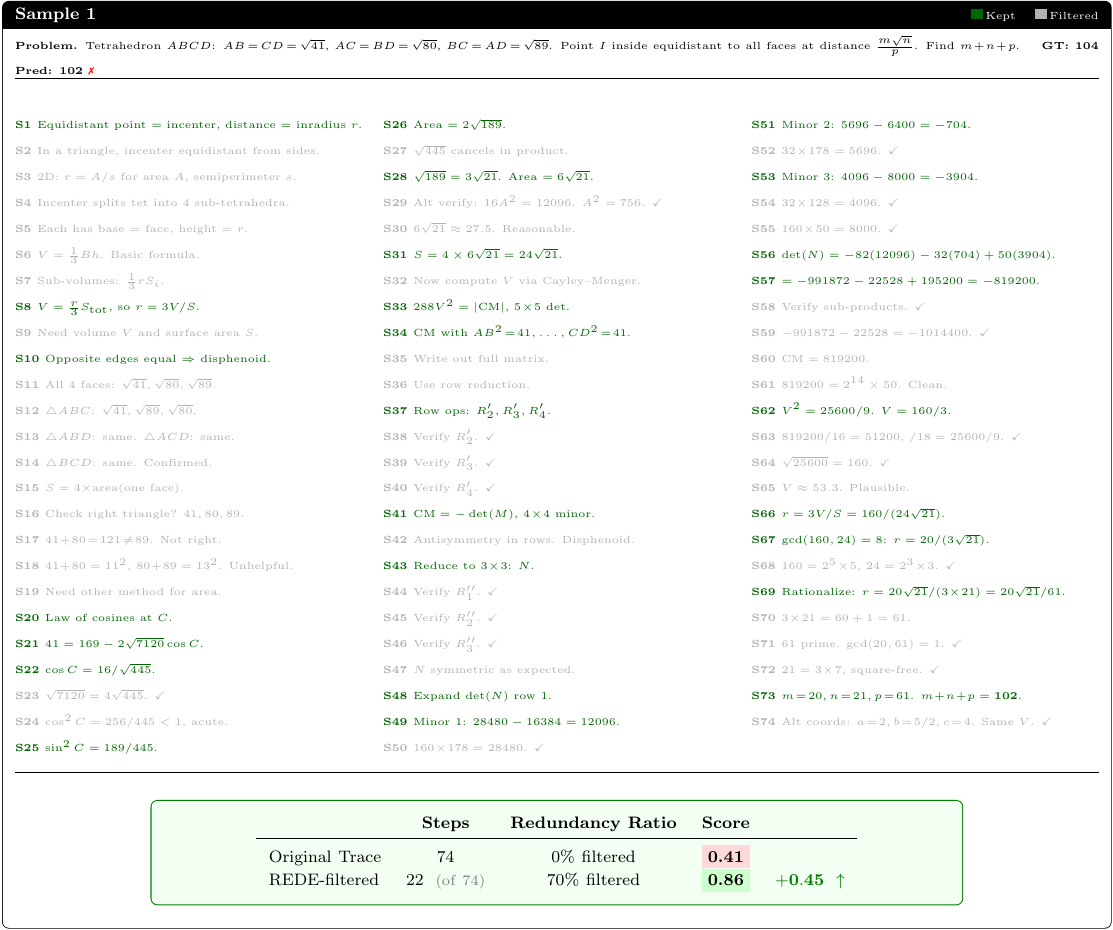}
\caption{\small \textbf{Sample~1 (incorrect answer).}
\model step-level filtering on a tetrahedron inradius problem from \textsc{Math} with Qwen3-8B.
\textcolor{green!40!black}{Green}: kept; \textcolor{black!30}{gray}: filtered.
Detection scores are from supervised probing~\citep{azaria2023internal}.
After filtering 70\% of the steps, the score increases from 0.41 to 0.86, correctly flagging the hallucinated response.}
\label{fig:rede_sample1}
\end{figure}

\begin{figure}[h]
\centering
\includegraphics[width=\linewidth]{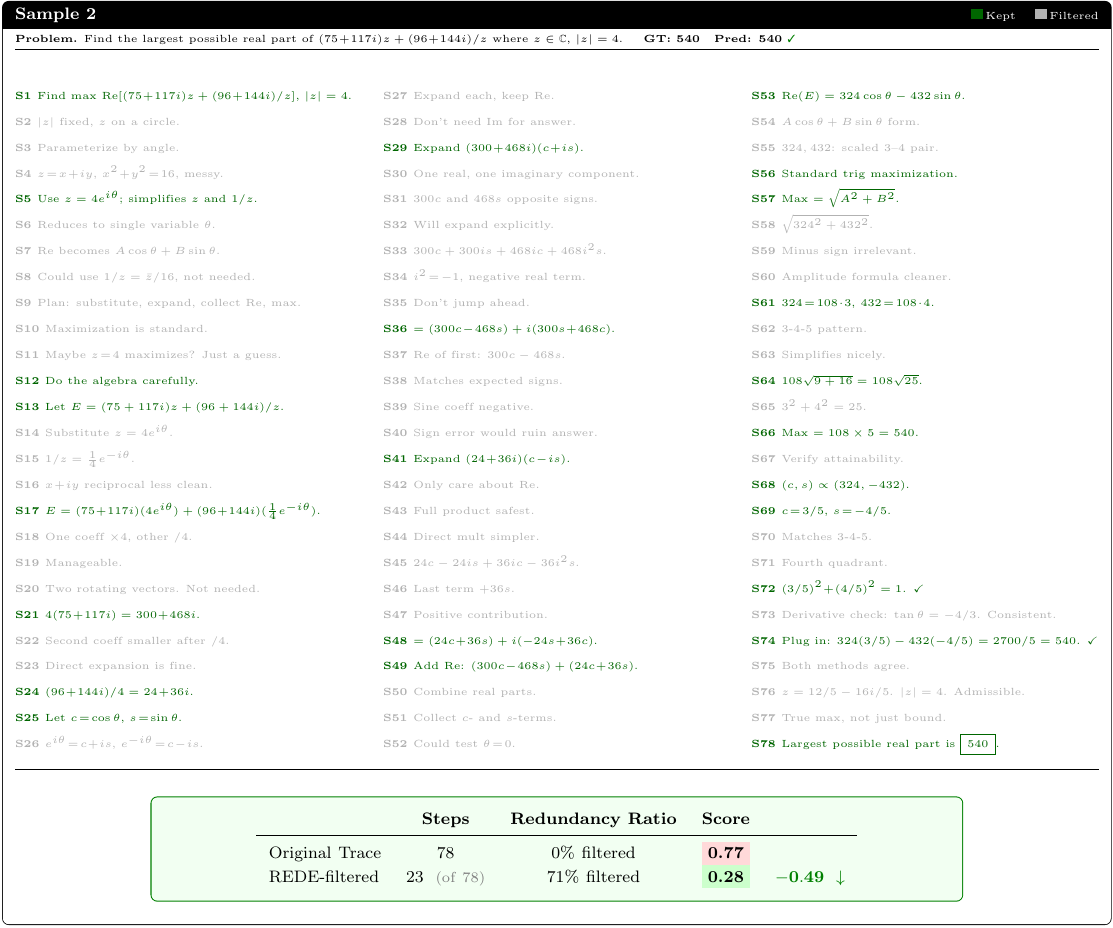}
\caption{\small \textbf{Sample~2 (correct answer).}
\model step-level filtering on a complex-number optimization problem from \textsc{Math} with Qwen3-8B.
\textcolor{green!40!black}{Green}: kept; \textcolor{black!30}{gray}: filtered.
Detection scores are from supervised probing~\citep{azaria2023internal}.
After filtering 71\% of the steps, the score decreases from 0.77 to 0.28, correctly reducing the false-positive rate.}
\label{fig:rede_sample2}
\end{figure}

\noindent \textbf{Answer embedding visualization.}
We further examine whether step-level denoising improves the separability of truthful and hallucinated traces. Figure~\ref{fig:trace_embedding} shows the answer-level embeddings on TruthfulQA with Qwen3-8B. In the original space (left), the two classes are heavily mixed. After \model filtering (right), they become more clearly separated, with only limited overlap near the decision boundary. This is consistent with Proposition~\ref{prop:noise_hurts} in Appendix~\ref{app:theory}, which suggests that removing noisy steps improves downstream discriminability by reducing within-class variance.

\begin{figure}[h]
\centering
\includegraphics[width=0.8\linewidth]{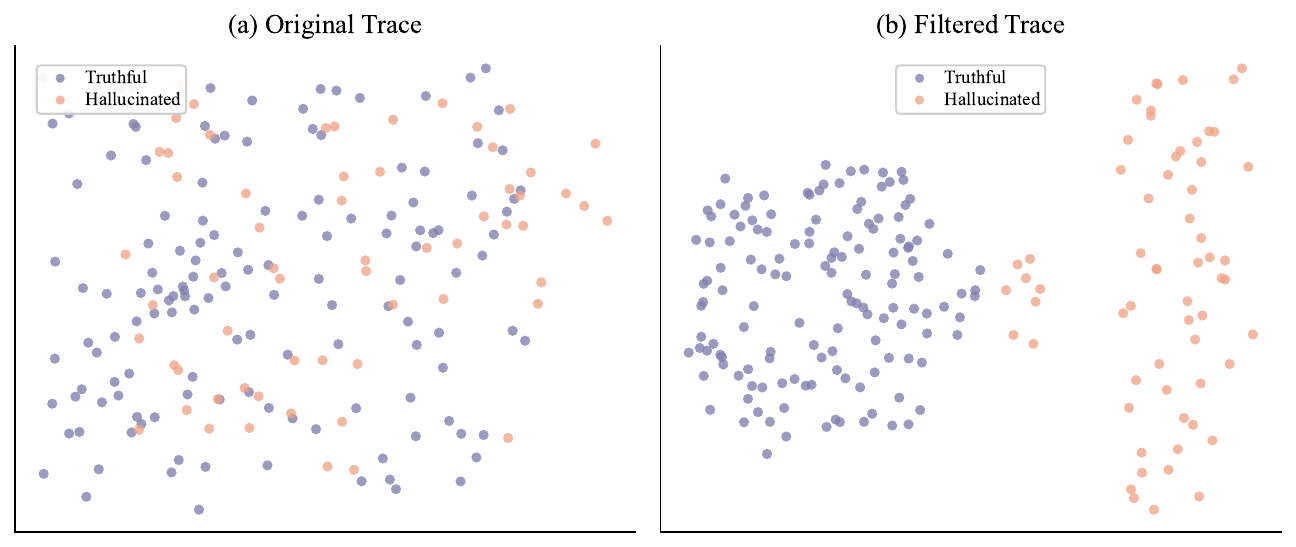}
\caption{\small Answer embeddings on TruthfulQA with Qwen3-8B via supervised probing~\citep{azaria2023internal}. Left: original full traces, where truthful and hallucinated responses overlap heavily. Right: after \model denoising, the two groups become clearly more separable.}
\label{fig:trace_embedding}
\end{figure}

\subsection{Additional results on larger LRMs}
\label{app:large_lrm_more_results}

Figure~\ref{fig:generalizability}(b) demonstrates the scalability of \model on TruthfulQA. Here we extend this evaluation to the other three benchmarks. As shown in Table~\ref{tab:large_lrm_more_results}, \model consistently outperforms two competitive LRM-specific baselines, RACE~\citep{wang2026joint} and ARS~\citep{zhang2026harnessing}, across all benchmarks on both Qwen3-32B and DeepSeek-R1-Distill-Qwen-32B. For example, compared with ARS, \model improves by 14.63\%, 4.27\%, and 8.41\% on MATH, CodeElo, and MULTIHOPQA with Qwen3-32B, respectively. Similar trends hold on DeepSeek-R1-Distill-Qwen-32B, where \model surpasses ARS by 2.99\% on MATH, 3.55\% on CodeElo, and 5.47\% on MULTIHOPQA. These results confirm that reasoning-step denoising remains effective across diverse reasoning tasks as LRMs scale up.

\begin{table}[t]
\centering
\small
\setlength{\tabcolsep}{7pt}
\caption{\small Additional results on larger LRMs across three benchmarks. All values are AUROC with supervised probing as the downstream detector. The best result in each column is highlighted in \textbf{bold}.}
\label{tab:large_lrm_more_results}
\resizebox{0.8\textwidth}{!}{%
\begin{tabular}{lcccccc}
\toprule
\multirow{2}{*}{Method} & \multicolumn{3}{c}{Qwen3-32B} & \multicolumn{3}{c}{DeepSeek-R1-Distill-Qwen-32B} \\
\cmidrule(lr){2-4} \cmidrule(lr){5-7}
& MATH & CodeElo & MULTIHOPQA & MATH & CodeElo & MULTIHOPQA \\
\midrule
RACE~\citep{wang2026joint} & 77.14 & 80.96 & 77.83 & 67.66 & 63.72 & 67.91 \\
ARS~\citep{zhang2026harnessing} & 70.09 & 78.94 & 73.26 & 82.44 & 78.63 & 70.47 \\
\model & \textbf{84.72} & \textbf{83.21} & \textbf{81.67} & \textbf{85.43} & \textbf{82.18} & \textbf{75.94} \\
\bottomrule
\end{tabular}%
}
\end{table}

\subsection{Effect of sampling strategies}
\label{app:sampling}
We evaluate hallucination detection performance when \model generates reasoning traces under different sampling strategies.
Our main results are obtained with greedy decoding, which selects the most likely next token at each step.
We additionally compare it with multinomial sampling using a temperature of 0.5.
As shown in Table~\ref{tab:ablation_sampling}, greedy decoding consistently outperforms multinomial sampling across all detectors.
The improvement is particularly notable for the Perplexity and Verbalized detectors, suggesting that more deterministic reasoning traces provide more reliable signals for downstream hallucination detection.

\begin{table}[h]
\centering
\caption{\small Hallucination detection results under different sampling strategies on TruthfulQA with Qwen3-8B. The best result in each column is highlighted in \textbf{bold}.}
\label{tab:ablation_sampling}
\small
\resizebox{0.8\linewidth}{!}{%
\begin{tabular}{lcccc}
\toprule
\textbf{Sampling Strategy}
& \textbf{CCS}~\citep{burnsdiscovering}
& \textbf{Probing}~\citep{azaria2023internal}
& \textbf{Perplexity}~\citep{ren2023outofdistribution}
& \textbf{Verbalized}~\citep{linteaching} \\
\midrule
Multinomial sampling  & 86.49 & 85.07 & 58.11 & 57.14 \\
Greedy decoding (Ours) & \textbf{87.32} & \textbf{86.44} & \textbf{63.46} & \textbf{61.10} \\
\bottomrule
\end{tabular}%
}
\end{table}

\subsection{Accuracy of step selection}
\label{app:step_acc}
We further evaluate whether \model can accurately identify noisy steps. On AIME 2024~\citep{aime24} with Qwen3-8B, we compare the steps flagged as noisy by \model against human annotations that categorize each step as informative, irrelevant, or repetitive. \model achieves an accuracy of 86.37\% in classifying steps as informative versus noisy, confirming that the learned projection reliably identifies genuinely noisy steps and that the downstream detection gains stem from accurate step filtering.

\subsection{Effect of truthfulness labeling methods}
\label{app:labeling}
In our main experiments, the correctness of model-generated answers is judged using a strong external model Qwen3-32B.
In this ablation, we show that the results remain robust under different judgment methods, including ROUGE-L, BLEURT~\citep{sellam2020bleurt}, and a different judge model, DeepSeek-R1-Distill-Qwen-32B~\citep{guo2025deepseek}.
For ROUGE-L, a generation is deemed truthful when the similarity score between the generation and the ground truth exceeds a threshold of 0.3.
For BLEURT, we use the \texttt{bleurt-base-128} variant, a learned metric built upon BERT~\citep{devlin2019bert} and augmented with diverse lexical and semantic-level supervision signals.
As shown in Table~\ref{tab:ableling}, \model consistently achieves strong detection performance across all labeling methods, confirming that our approach does not rely on a specific judge model or evaluation metric.

\begin{table}[h]
\centering
\caption{\small Hallucination detection results (AUROC, \%) under different truthfulness labeling methods on TruthfulQA and \textsc{Math}. For \model, we report results using probing as the downstream detector. The best result in each column is highlighted in \textbf{bold}.}
\label{tab:ableling}
\small
\setlength{\tabcolsep}{3.5pt}
\resizebox{0.8\columnwidth}{!}{
\begin{tabular}{lcccccc}
\toprule
\multirow{2}{*}{\textbf{Labeling Method}}
& \multicolumn{2}{c}{\textbf{ROUGE-L}}
& \multicolumn{2}{c}{\textbf{BLEURT}}
& \multicolumn{2}{c}{\textbf{DeepSeek-R1-Distill-Qwen-32B}} \\
\cmidrule(lr){2-3} \cmidrule(lr){4-5} \cmidrule(lr){6-7}
& TruthfulQA & \textsc{Math}
& TruthfulQA & \textsc{Math}
& TruthfulQA & \textsc{Math} \\
\midrule
RACE~\citep{wang2026joint}
& 85.82 & 77.34 & 86.21 & 76.95 & 65.21 & 69.02 \\

ARS (CCS)~\citep{zhang2026harnessing}
& 93.56 & 81.71 & 90.18 & 84.75 & 80.07 & 65.49 \\

\model (Probing)
& \textbf{95.08} & \textbf{83.77} & \textbf{92.77} & \textbf{85.04} & \textbf{85.13} & \textbf{78.41} \\
\bottomrule
\end{tabular}
}
\end{table}

\subsection{Results on an additional dataset}
\label{app:gsm8k}
To further evaluate the effectiveness of \model beyond the datasets used in the main paper, we conduct experiments on GSM8K~\citep{cobbe2021training}. For mathematical reasoning, we use the train split, which contains 7,473 problems.
Following the same experimental setup as in our main experiments, we reserve 25\% of the data for testing, use 100 examples for validation, and use the remaining examples for training.
As shown in Table~\ref{tab:gsm8k}, using \model-selected steps consistently outperforms using the original reasoning trace across all detectors and both model families.
The gains are particularly substantial for CCS and remain consistent for probing, perplexity, and verbalized confidence, showing that reasoning-step denoising remains effective on this additional mathematical reasoning dataset.

\begin{table}[h]
\centering
\caption{\small \textbf{Comparison of hallucination detection using the original reasoning trace versus \model-selected steps on GSM8K.} All values are percentages (AUROC). The best results are highlighted in \textbf{bold}.}
\label{tab:gsm8k}
\small
\setlength{\tabcolsep}{4pt}
\renewcommand{\arraystretch}{1.08}
\resizebox{0.8\columnwidth}{!}{%
\begin{tabular}{lcccccccc}
\toprule
\multirow{2}{*}{\textbf{Model}} &
\multicolumn{2}{c}{\textbf{CCS}~\citep{burnsdiscovering}} &
\multicolumn{2}{c}{\textbf{Probing}~\citep{azaria2023internal}} &
\multicolumn{2}{c}{\textbf{Perplexity}~\citep{ren2023outofdistribution}} &
\multicolumn{2}{c}{\textbf{Verbalized}~\citep{linteaching}} \\
\cmidrule(lr){2-3} \cmidrule(lr){4-5} \cmidrule(lr){6-7} \cmidrule(lr){8-9}
& Original & Filtered & Original & Filtered & Original & Filtered & Original & Filtered \\
\midrule
Qwen3-8B 
& 60.08 & \textbf{91.01} & 87.50 & \textbf{92.44} & 55.21 & \textbf{60.03} & 50.07 & \textbf{60.09} \\
DeepSeek-R1-Distill-Llama-8B 
& 58.31 & \textbf{84.22} & 78.61 & \textbf{84.66} & 52.31 & \textbf{58.07} & 53.81 & \textbf{62.30} \\
\bottomrule
\end{tabular}%
}
\end{table}

\subsection{Validating the attention signal}
\label{app:attention_validation}
We further investigate (i) whether final-answer attention reflects a meaningful notion of step usefulness, and (ii) the advantage of \model over simple position-based filtering strategies. Experiments are conducted on all four benchmarks with Qwen3-8B, using supervised probing as the downstream detector. All filtering operations use the same drop ratio $\zeta = 70\%$ as \model.

\noindent \textbf{Final-answer attention reflects step usefulness.} We compare three filtering strategies against the original unfiltered trace: dropping the top-$70\%$ and the bottom-$70\%$ of reasoning steps ranked by their final-answer attention score. For each strategy, we evaluate the cosine similarity between the answer embedding of the remaining trace and that of the full trace, together with the downstream detection AUROC. As shown in Table~\ref{tab:attention_drop}, dropping high-attention steps substantially degrades both metrics compared with the original trace (e.g., AUROC drops from 80.42 to 57.18 on TruthfulQA), whereas dropping low-attention steps preserves and often improves detection quality (e.g., 84.97 vs. 80.42 on TruthfulQA) while maintaining a high cosine similarity to the full-trace answer embedding. This asymmetry shows that final-answer attention reliably captures step-level usefulness for hallucination detection: high-attention steps carry information essential to the final answer, while low-attention steps are largely dispensable.

\noindent \textbf{Comparison with position-based filtering strategies.} We compare \model against two position-based baselines under the same drop ratio: Drop-Latest removes the latest $70\%$ of steps, and Drop-Earliest removes the earliest $70\%$ of steps. As shown in Table~\ref{tab:position_baselines}, Drop-Latest performs worst across all datasets, and while Drop-Earliest is stronger, it still falls well below \model (e.g., 86.44\% vs. 77.53\% on TruthfulQA). This demonstrates that our attention-based supervision captures step-level usefulness beyond simple positional patterns, and \model effectively leverages this signal to identify informative steps throughout the reasoning trace.

\begin{table}[h]
\centering
\caption{Effect of dropping steps ranked by final-answer attention score on Qwen3-8B, compared against the original unfiltered trace. ``Cos Sim'' denotes the cosine similarity between the answer embedding of the remaining trace and that of the full trace (1.000 for the original trace by definition); AUROC is computed with supervised probing~\citep{azaria2023internal}.}
\label{tab:attention_drop}
\resizebox{0.8\textwidth}{!}{%
\begin{tabular}{lcccccccc}
\toprule
 & \multicolumn{2}{c}{TruthfulQA} & \multicolumn{2}{c}{MATH} & \multicolumn{2}{c}{CodeElo} & \multicolumn{2}{c}{\textsc{MULTIHOPQA}} \\
\cmidrule(lr){2-3} \cmidrule(lr){4-5} \cmidrule(lr){6-7} \cmidrule(lr){8-9}
Method & Cos Sim & AUROC & Cos Sim & AUROC & Cos Sim & AUROC & Cos Sim & AUROC \\
\midrule
Original trace (no filtering) & 1.000 & 80.42 & 1.000 & 81.05 & 1.000 & 82.82 & 1.000 & 77.84 \\
Drop top-$70\%$ attention    & 0.427 & 57.18 & 0.361 & 52.84 & 0.483 & 61.07 & 0.394 & 50.62 \\
Drop bottom-$70\%$ attention & 0.872 & 84.97 & 0.815 & 83.62 & 0.846 & 85.11 & 0.829 & 79.48 \\
\bottomrule
\end{tabular}%
}
\end{table}

\begin{table}[h]
\centering
\caption{Comparison of \model against position-based filtering baselines, all using drop ratio $\zeta = 70\%$ on Qwen3-8B. Drop-Latest removes the latest $70\%$ of steps (retaining the earliest $30\%$), and Drop-Earliest removes the earliest $70\%$ of steps (retaining the latest $30\%$). AUROC is computed with supervised probing~\citep{azaria2023internal}. The best result in each column is highlighted in \textbf{bold}.}
\label{tab:position_baselines}
\resizebox{0.6\textwidth}{!}{%
\begin{tabular}{lcccc}
\toprule
Method & TruthfulQA & MATH & CodeElo & \textsc{MULTIHOPQA} \\
\midrule
Drop-Latest          & 68.42 & 65.71 & 70.83 & 66.94 \\
Drop-Earliest        & 77.53 & 75.68 & 79.14 & 74.29 \\
\model (Ours)        & \textbf{86.44} & \textbf{87.06} & \textbf{87.19} & \textbf{82.94} \\
\bottomrule
\end{tabular}%
}
\end{table}

\subsection{Applicability in a black-box setting}
\label{app:blackbox}

Our main setup assumes white-box access to the target LRM in order to compute final-answer attention and step embeddings. We further study whether \model can be applied when the target LRM is a black box, by using a separate white-box LRM as a proxy. Specifically, for each input generated by the black-box target, we feed the same prompt and its generated reasoning trace into the proxy to compute the final-answer attention and step embeddings, and then apply \model to obtain the filtered trace. The filtered trace is passed back to the target model for verbalized certainty estimation, which does not require internal access. We use Qwen3-32B~\citep{yang2025qwen3} as the proxy and evaluate on two target models: DeepSeek-R1-Distill-Llama-8B~\citep{guo2025deepseek} and Qwen3-8B~\citep{yang2025qwen3}.

As shown in Table~\ref{tab:blackbox}, the proxy-based variant consistently improves detection over the original trace across all four benchmarks and both target models, improving over the original by 5.92\% on MultiHopQA with DeepSeek and 8.04\% on TruthfulQA with Qwen3-8B, recovering most of the gain obtained with white-box \model. This suggests that \model remains effective in black-box settings where a suitable proxy LRM is available.

\begin{table}[h]
\centering
\caption{Applicability of \model in a black-box setting with Qwen3-32B as the white-box proxy. Detection uses verbalized certainty, which does not require internal access to the target. ``\model (proxy)'' uses Qwen3-32B for attention scoring and step embedding, while ``\model (white-box)'' has direct access to the target and is shown as an upper-bound reference. All values are AUROC (\%). The best result among black-box methods is highlighted in \textbf{bold}.}
\label{tab:blackbox}
\resizebox{\textwidth}{!}{%
\begin{tabular}{lcccccccc}
\toprule
 & \multicolumn{4}{c}{DeepSeek-R1-Distill-Llama-8B} & \multicolumn{4}{c}{Qwen3-8B} \\
\cmidrule(lr){2-5} \cmidrule(lr){6-9}
Method & TruthfulQA & MATH & CodeElo & \textsc{MULTIHOPQA} & TruthfulQA & MATH & CodeElo & \textsc{MULTIHOPQA} \\
\midrule
Original trace     & 56.61 & 69.71 & 79.13 & 50.87 & 50.37 & 57.19 & 78.89 & 63.75 \\
\model (proxy)     & \textbf{59.85} & \textbf{74.26} & \textbf{82.07} & \textbf{56.79} & \textbf{58.41} & \textbf{62.83} & \textbf{82.14} & \textbf{67.62} \\
\midrule
\rowcolor{gray!15}
\model (white-box) & 61.94 & 76.13 & 84.49 & 59.56 & 61.10 & 65.13 & 83.89 & 69.50 \\
\bottomrule
\end{tabular}%
}
\end{table}

\subsection{Impact of step filtering on downstream task accuracy}
\label{appendix:filtering_accuracy}

Our main experiments measure hallucination detection performance (AUROC). Here we examine a complementary aspect: whether the filtered reasoning trace preserves the LRM's ability to produce correct answers. This evaluation targets \emph{task-level accuracy}, which is distinct from hallucination detection. Concretely, for each reasoning trace generated by Qwen3-8B on \textsc{MATH}, we remove the steps identified as noisy by \model and regenerate the final answer conditioned on the retained steps. Task accuracy is evaluated using the same judge model as in our main experiments. After filtering, the reasoning traces are reduced by approximately 70\% in length. Despite this substantial reduction, the filtered traces achieve an accuracy of 80.00\%, compared to 81.43\% for the original traces, maintaining a comparable level of task performance. This confirms that the vast majority of discarded steps are uninformative for solving the task, and that \model's filtering preserves the reasoning capability of the LRM even with a significantly compressed trace.

\subsection{Results with additional metrics}
\label{app:metrics}
Our main experiments use AUROC as the primary evaluation metric, as it measures discrimination ability across all decision thresholds.
To provide a more complete evaluation, we additionally report Accuracy and F1 score.
As shown in Table~\ref{tab:ablation_metrics}, using \model-Filtered steps consistently improves over using the Original reasoning trace across all three metrics for both CCS and probing.
Notably, beyond the substantial gains in AUROC, \model also achieves strong improvements in ACC and F1, further confirming the effectiveness of reasoning-step denoising.

\begin{table}[h]
\centering
\caption{\small \textbf{Hallucination detection results under additional evaluation metrics on TruthfulQA with Qwen3-8B.} We compare using the Original reasoning trace versus \model-Filtered steps. The best result in each group is highlighted in \textbf{bold}.}
\label{tab:ablation_metrics}
\small
\setlength{\tabcolsep}{5pt}
\renewcommand{\arraystretch}{1.08}
\resizebox{0.5\columnwidth}{!}{%
\begin{tabular}{lcccccc}
\toprule
\multirow{2}{*}{\textbf{Setting}}
& \multicolumn{3}{c}{\textbf{CCS}~\citep{burnsdiscovering}}
& \multicolumn{3}{c}{\textbf{Probing}~\citep{azaria2023internal}} \\
\cmidrule(lr){2-4} \cmidrule(lr){5-7}
& AUROC & ACC & F1 & AUROC & ACC & F1 \\
\midrule
Original   
& 68.63 & 58.91 & 65.03 
& 80.42 & 66.57 & 63.08 \\
Filtered 
& \textbf{87.32} & \textbf{70.47} & \textbf{73.95} 
& \textbf{86.44} & \textbf{71.09} & \textbf{68.82} \\
\bottomrule
\end{tabular}%
}
\end{table}

\section{Broader impact and limitations}
\label{sec:imapct_and_limitation}
\textbf{Broader Impact.} Large reasoning models (LRMs) are increasingly used in applications where users may rely on the factuality of generated answers. Our work proposes \model, a lightweight framework that denoises reasoning traces for hallucination detection, and we expect it to have a positive impact on the reliability of LRM-based systems, such as question answering, educational assistants, and coding support. By filtering noisy reasoning steps before downstream detection, \model can help flag potentially unreliable outputs prior to user-facing deployment. At the same time, \model is not a guarantee of truthfulness. Like other detection methods, it may still make incorrect predictions, which could either over-trust hallucinated answers or suppress correct but unconventional reasoning traces. Therefore, we view \model as a reliability-enhancing component rather than a substitute for human oversight, especially in high-stakes settings.

\noindent \textbf{Limitations.} \model currently operates as a filtering module on reasoning traces. It does not intervene during the generation process itself. A promising direction for future work is to integrate the denoising signal into the LRM, for example by incorporating the step-level informativeness objective into model training so that the LRM learns to suppress noisy steps during generation.

\newpage 
\section{Theoretical analysis}
\label{app:theory}

In this section, we provide a theoretical perspective on why \model improves hallucination detection.
Our analysis follows the practical inference rule in Section~\ref{sec:inference}, where steps are filtered by their intra-trace $\kappa$-NN distances in the shaped space.
At a high level, we formalize the following story:
(1) the oracle informative trace provides a more discriminative representation for hallucination detection than the full trace containing noisy steps;
(2) after representation shaping, informative steps become locally dense while noisy steps remain locally sparse, which separates their $\kappa$-NN distance scores;
(3) under rank-based removal of the top-$\zeta\%$ largest $\kappa$-NN distances, the resulting filtering errors are small; and
(4) small filtering errors imply that the hallucination detection risk on the filtered trace is close to that on the oracle informative trace.

To keep the analysis analytically tractable, we introduce in Assumption~\ref{assump:label_uninform} an $\varepsilon$-approximate form of conditional label-uninformativeness of the noisy perturbation, parameterized by two scalar quantities $(\varepsilon_\mu,\varepsilon_c)$. Proposition~\ref{prop:noise_hurts} is correspondingly stated as an additive perturbation bound whose slack decays continuously with $(\varepsilon_\mu,\varepsilon_c)$ and vanishes in the strict-independence limit, so that our bounds remain meaningful across a broad range of noisy-step behaviors.

\subsection{Setup and definitions}

Consider a test-time prompt-answer pair $(\bar{\mathbf{p}}, \bar{\mathbf{a}})$ with reasoning trace
$
\bar{\mathbf{C}}=\{\bar{\mathbf{c}}_1,\ldots,\bar{\mathbf{c}}_{\bar{K}}\},
$
following the test-time notation in Section~\ref{sec:inference}.
Let $\bar{\mathbf{C}}^{\mathrm{info}} \subseteq \bar{\mathbf{C}}$ denote the subset of informative reasoning steps, and let
$
\bar{\mathbf{C}}^{\mathrm{noise}} = \bar{\mathbf{C}}\setminus \bar{\mathbf{C}}^{\mathrm{info}}
$
denote the subset of noisy steps.
At test time, \model produces a filtered trace $\bar{\mathbf{C}}^{\mathrm{fil}} \subseteq \bar{\mathbf{C}}$ by removing steps that are identified as noisy in the learned representation space.

Let $\boldsymbol{\psi}(\bar{\mathbf{c}}) \in \mathbb{R}^{d_{\Phi}}$ denote the representation of reasoning step $\bar{\mathbf{c}}$ used by the downstream hallucination detector, and let the trace-level representation be
\begin{equation}
\boldsymbol{\Phi}(\bar{\mathbf{p}}, \bar{\mathbf{C}}, \bar{\mathbf{a}})
=
\boldsymbol{\Phi}_0(\bar{\mathbf{p}}, \bar{\mathbf{a}})
+
\frac{1}{\bar{K}}\sum_{\bar{\mathbf{c}}\in \bar{\mathbf{C}}} \boldsymbol{\psi}(\bar{\mathbf{c}}),
\label{eq:theory_trace_repr}
\end{equation}
where $\boldsymbol{\Phi}_0(\bar{\mathbf{p}}, \bar{\mathbf{a}})$ captures the prompt-answer component independent of the reasoning steps.
This form covers common downstream detectors that operate on pooled hidden states or other additive trace summaries.
We normalize by the original trace length $\bar{K}$ so that filtered and oracle traces are compared under a common scale.

Let $g:\mathbb{R}^{d_{\Phi}}\to\mathbb{R}$ be a downstream detector, and let $\ell(g(\boldsymbol{\Phi}),y)$ be the detection loss for truthfulness label $y\in\{0,1\}$.
We define the population risk on a trace subset $\bar{\mathbf{C}}' \subseteq \bar{\mathbf{C}}$ as
\begin{equation}
\mathcal{R}(g;\bar{\mathbf{C}}')
=
\mathbb{E}\big[\ell(g(\boldsymbol{\Phi}(\bar{\mathbf{p}},\bar{\mathbf{C}}',\bar{\mathbf{a}})),y)\big],
\label{eq:theory_risk}
\end{equation}
where the expectation is taken over the data distribution.

To quantify the quality of filtering, we define two error rates:
\begin{align}
\alpha
&:=
\mathbb{E}\!\left[
\frac{\left|\bar{\mathbf{C}}^{\mathrm{fil}}\cap \bar{\mathbf{C}}^{\mathrm{noise}}\right|}{\bar{K}}
\right],
\label{eq:theory_alpha}
\\[4pt]
\beta
&:=
\mathbb{E}\!\left[
\frac{\left|\bar{\mathbf{C}}^{\mathrm{info}}\setminus \bar{\mathbf{C}}^{\mathrm{fil}}\right|}{\bar{K}}
\right].
\label{eq:theory_beta}
\end{align}
Here, $\alpha$ measures the fraction of noisy steps that are erroneously retained, while $\beta$ measures the fraction of informative steps that are mistakenly removed; both expectations are taken over the data distribution, including the randomness of the trace length $\bar{K}$.

Let $\bar{\mathbf{e}}_i$ denote the pre-projection embedding of step $\bar{\mathbf{c}}_i$, and let
\begin{equation}
\bar{\mathbf{z}}_i = f_\phi(\bar{\mathbf{e}}_i) \in \mathbb{R}^{d'}
\label{eq:theory_projected_step}
\end{equation}
be the projected step embedding produced by the learned projection head $f_\phi$.
Consistent with the practical inference rule in Section~\ref{sec:inference} and the cosine-similarity-based training objective in Eq.~\ref{eq:total_loss}, throughout this appendix we instantiate the distance over the projected space as the cosine distance
\begin{equation}
D(\mathbf{u},\mathbf{v})
:=
1 - \frac{\mathbf{u}^\top\mathbf{v}}{\|\mathbf{u}\|_2\,\|\mathbf{v}\|_2},
\label{eq:theory_cosine}
\end{equation}
matching the score $S_i$ used at inference in Section~\ref{sec:inference}.
To match this practical inference rule, we define the step-level filtering score by the intra-trace $\kappa$-NN cosine distance
\begin{equation}
S_i
:=
D(\bar{\mathbf{z}}_i,\bar{\mathbf{z}}_i^{(\kappa)}),
\label{eq:theory_knn_score}
\end{equation}
where $\bar{\mathbf{z}}_i^{(\kappa)}$ denotes the $\kappa$-th nearest neighbor of $\bar{\mathbf{z}}_i$ among $\{\bar{\mathbf{z}}_j\}_{j\neq i}$ within the same trace under $D$.
We use $\kappa$ here to denote the nearest-neighbor parameter and reserve $\bar{K}$ for the number of steps in the test trace; correspondingly, the number of nearest neighbors denoted by $k$ in Section~\ref{sec:inference} is written as $\kappa$ throughout this appendix.

Following the practical rank-based rule, let
\begin{equation}
M := \lceil \zeta \bar{K} \rceil
\label{eq:theory_removed_steps}
\end{equation}
denote the number of removed steps, where $\zeta\in(0,1)$ is the removal ratio.
The filtered trace $\bar{\mathbf{C}}^{\mathrm{fil}}$ is obtained by removing the $M$ steps with the largest scores $\{S_i\}_{i=1}^{\bar{K}}$.

For a radius $r>0$, define the local neighbor count around step $i$ by
\begin{equation}
N_i(r)
:=
\sum_{j\neq i}
\mathbf{1}\!\left\{
D(\bar{\mathbf{z}}_j,\bar{\mathbf{z}}_i) \le r
\right\}.
\label{eq:theory_neighbor_count}
\end{equation}
Then the $\kappa$-NN score admits the equivalent characterization
\begin{equation}
S_i \le r
\iff
N_i(r)\ge \kappa,
\qquad
S_i > r
\iff
N_i(r)<\kappa.
\label{eq:knn_count_equiv}
\end{equation}

\subsection{Oracle informative traces are more discriminative}
\label{subsec:oracle_discriminative}

We first formalize why noisy reasoning steps are harmful for hallucination detection.
The key intuition is that noisy steps act as approximately label-uninformative perturbations to the trace representation: they do not substantially improve the class-conditional mean separation relevant for truthfulness prediction, but they increase the within-class variability, thereby reducing class separability.

Let
\[
\boldsymbol{\Phi}^{\star} := \boldsymbol{\Phi}(\bar{\mathbf{p}},\bar{\mathbf{C}}^{\mathrm{info}},\bar{\mathbf{a}}),
\qquad
\boldsymbol{\delta} := \frac{1}{\bar{K}}\sum_{\bar{\mathbf{c}}\in \bar{\mathbf{C}}^{\mathrm{noise}}}\boldsymbol{\psi}(\bar{\mathbf{c}}),
\]
so that the full-trace representation satisfies
\begin{equation}
\boldsymbol{\Phi}(\bar{\mathbf{p}},\bar{\mathbf{C}},\bar{\mathbf{a}}) = \boldsymbol{\Phi}^{\star} + \boldsymbol{\delta}.
\label{eq:phi_decompose}
\end{equation}

For each label $y\in\{0,1\}$, define the class-conditional means and covariances
\[
\boldsymbol{\mu}_y^\star := \mathbb{E}[\boldsymbol{\Phi}^{\star} \mid y],
\qquad
\boldsymbol{\Sigma}_y^\star := \mathrm{Cov}(\boldsymbol{\Phi}^{\star} \mid y),
\]
and
\[
\boldsymbol{\mu}_y := \mathbb{E}[\boldsymbol{\Phi} \mid y],
\qquad
\boldsymbol{\Sigma}_y := \mathrm{Cov}(\boldsymbol{\Phi} \mid y).
\]
Let the pooled within-class covariance matrices be
\[
\boldsymbol{\Sigma}_w^\star := \pi_0 \boldsymbol{\Sigma}_0^\star + \pi_1 \boldsymbol{\Sigma}_1^\star,
\qquad
\boldsymbol{\Sigma}_w := \pi_0 \boldsymbol{\Sigma}_0 + \pi_1 \boldsymbol{\Sigma}_1,
\]
where $\pi_y := \mathbb{P}(Y=y)$.

\begin{assumption}[$\varepsilon$-approximate label-uninformative perturbation]
\label{assump:label_uninform}
There exist constants $\varepsilon_\mu, \varepsilon_c \ge 0$ such that for each $y\in\{0,1\}$:
\begin{enumerate}
    \item[(i)] (Approximate mean-zero)\quad $\big\|\mathbb{E}[\boldsymbol{\delta} \mid y]\big\|_2 \le \varepsilon_\mu$;
    \item[(ii)] (Bounded cross-covariance)\quad $\big\|\mathrm{Cov}(\boldsymbol{\Phi}^{\star},\boldsymbol{\delta} \mid y)\big\|_{\mathrm{op}} \le \varepsilon_c$;
    \item[(iii)] $\mathrm{Cov}(\boldsymbol{\delta} \mid y) = \boldsymbol{\Sigma}_\delta(y) \succeq \mathbf{0}$, $\boldsymbol{\Sigma}_w^\star \succ \mathbf{0}$, and the cross-covariance slack is moderate relative to the informative covariance, i.e., $\varepsilon_c \le \lambda_{\min}(\boldsymbol{\Sigma}_w^\star)/4$.
\end{enumerate}
The strict-independence case $\boldsymbol{\delta}\perp\boldsymbol{\Phi}^{\star}\mid y$ together with $\mathbb{E}[\boldsymbol{\delta}\mid y]=\mathbf{0}$ is recovered as the special limit $\varepsilon_\mu=\varepsilon_c=0$.
\end{assumption}

\begin{remark}[Scope and tightness of Assumption~\ref{assump:label_uninform}]
\label{remark:assumption_limitation}
Assumption~\ref{assump:label_uninform} captures a broad class of noisy-step behaviors through two nonnegative slack parameters $(\varepsilon_\mu,\varepsilon_c)$, which control, respectively, the class-conditional mean of the noisy perturbation $\boldsymbol{\delta}$ and its second-order correlation with the informative component $\boldsymbol{\Phi}^{\star}$. The mild magnitude condition on $\varepsilon_c$ in~(iii) ensures that the perturbed within-class covariance $\boldsymbol{\Sigma}_w$ remains well-conditioned relative to its oracle counterpart $\boldsymbol{\Sigma}_w^\star$, which is needed for the matrix-inverse perturbation bound used in the proof of Proposition~\ref{prop:noise_hurts}.
The strict-independence case corresponds to the limit $\varepsilon_\mu=\varepsilon_c=0$, while general $(\varepsilon_\mu,\varepsilon_c)$ allow arbitrary nonzero slack within the stated bounds.
Proposition~\ref{prop:noise_hurts} below is an additive perturbation bound whose slack $\Delta(\varepsilon_\mu,\varepsilon_c)$ decays continuously with these parameters, so the tightness of the bound scales smoothly with how close the noisy perturbation is to being label-uninformative.
\end{remark}

We quantify class separability of a trace representation by the Fisher discriminability
\begin{equation}
\mathcal{J}(\bar{\mathbf{C}}')
:=
\big(\boldsymbol{\mu}_1(\bar{\mathbf{C}}')-\boldsymbol{\mu}_0(\bar{\mathbf{C}}')\big)^\top
\boldsymbol{\Sigma}_w(\bar{\mathbf{C}}')^{-1}
\big(\boldsymbol{\mu}_1(\bar{\mathbf{C}}')-\boldsymbol{\mu}_0(\bar{\mathbf{C}}')\big),
\label{eq:fisher}
\end{equation}
where
\[
\boldsymbol{\mu}_y(\bar{\mathbf{C}}')
:=
\mathbb{E}[\boldsymbol{\Phi}(\bar{\mathbf{p}},\bar{\mathbf{C}}',\bar{\mathbf{a}}) \mid y],
\qquad
\boldsymbol{\Sigma}_w(\bar{\mathbf{C}}')
:=
\pi_0 \boldsymbol{\Sigma}_0(\bar{\mathbf{C}}') + \pi_1 \boldsymbol{\Sigma}_1(\bar{\mathbf{C}}').
\]

\begin{proposition}[Noisy steps reduce trace discriminability, approximate version]
\label{prop:noise_hurts}
Under Assumption~\ref{assump:label_uninform}, the full-trace Fisher discriminability is bounded by the oracle one up to a slack controlled by $(\varepsilon_\mu,\varepsilon_c)$:
\begin{equation}
\mathcal{J}(\bar{\mathbf{C}})
\le
\mathcal{J}(\bar{\mathbf{C}}^{\mathrm{info}})
+
\Delta(\varepsilon_\mu,\varepsilon_c),
\label{eq:noise_hurts_approx}
\end{equation}
where
\begin{equation}
\Delta(\varepsilon_\mu,\varepsilon_c)
:=
\frac{4\,\varepsilon_\mu\,\|\boldsymbol{\mu}_1^\star-\boldsymbol{\mu}_0^\star\|_2 + 4\,\varepsilon_\mu^2}{\lambda_{\min}(\boldsymbol{\Sigma}_w^\star)}
+
\frac{4\,\varepsilon_c\,\big(\|\boldsymbol{\mu}_1^\star-\boldsymbol{\mu}_0^\star\|_2+2\varepsilon_\mu\big)^{2}}{\lambda_{\min}(\boldsymbol{\Sigma}_w^\star)^{2}}.
\label{eq:slack_def}
\end{equation}
In particular, $\Delta(\varepsilon_\mu,\varepsilon_c)\to 0$ as $\varepsilon_\mu,\varepsilon_c \to 0$, recovering the exact monotonicity
$\mathcal{J}(\bar{\mathbf{C}}) \le \mathcal{J}(\bar{\mathbf{C}}^{\mathrm{info}})$ in the strict-independence limit.
\end{proposition}

\begin{proof}
By Eq.~\ref{eq:phi_decompose} and Assumption~\ref{assump:label_uninform}(i),
\[
\boldsymbol{\mu}_y
=
\boldsymbol{\mu}_y^\star + \mathbb{E}[\boldsymbol{\delta} \mid y]
=
\boldsymbol{\mu}_y^\star + \boldsymbol{\eta}_y,
\qquad
\|\boldsymbol{\eta}_y\|_2\le \varepsilon_\mu,
\]
so that
\[
\|\boldsymbol{\mu}_1-\boldsymbol{\mu}_0\|_2
\le
\|\boldsymbol{\mu}_1^\star-\boldsymbol{\mu}_0^\star\|_2 + 2\varepsilon_\mu.
\]
For the covariances, write
\[
\boldsymbol{\Sigma}_y
=
\boldsymbol{\Sigma}_y^\star + \boldsymbol{\Sigma}_\delta(y) + \mathbf{E}_y,
\qquad
\mathbf{E}_y := \mathrm{Cov}(\boldsymbol{\Phi}^{\star},\boldsymbol{\delta}\mid y) + \mathrm{Cov}(\boldsymbol{\delta},\boldsymbol{\Phi}^{\star}\mid y),
\]
with $\|\mathbf{E}_y\|_{\mathrm{op}}\le 2\varepsilon_c$ by Assumption~\ref{assump:label_uninform}(ii).
Therefore
\[
\boldsymbol{\Sigma}_w
=
\boldsymbol{\Sigma}_w^\star + \bar{\boldsymbol{\Sigma}}_\delta + \bar{\mathbf{E}},
\qquad
\bar{\boldsymbol{\Sigma}}_\delta := \pi_0\boldsymbol{\Sigma}_\delta(0)+\pi_1\boldsymbol{\Sigma}_\delta(1)\succeq \mathbf{0},
\quad
\|\bar{\mathbf{E}}\|_{\mathrm{op}}\le 2\varepsilon_c.
\]
Let $\mathbf{A}:=\boldsymbol{\Sigma}_w^\star+\bar{\boldsymbol{\Sigma}}_\delta \succeq \boldsymbol{\Sigma}_w^\star \succ \mathbf{0}$, so $\boldsymbol{\Sigma}_w=\mathbf{A}+\bar{\mathbf{E}}$.
By operator monotonicity of inversion on the positive-definite cone, $\mathbf{A}^{-1}\preceq (\boldsymbol{\Sigma}_w^\star)^{-1}$.
Using the standard perturbation bound for matrix inversion~\citep{horn2012matrix},
\[
\big\|\boldsymbol{\Sigma}_w^{-1}-\mathbf{A}^{-1}\big\|_{\mathrm{op}}
\le
\frac{\|\bar{\mathbf{E}}\|_{\mathrm{op}}}{\lambda_{\min}(\mathbf{A})\,(\lambda_{\min}(\mathbf{A})-\|\bar{\mathbf{E}}\|_{\mathrm{op}})}.
\]
Since $\lambda_{\min}(\mathbf{A}) \ge \lambda_{\min}(\boldsymbol{\Sigma}_w^\star)$ and $\|\bar{\mathbf{E}}\|_{\mathrm{op}}\le 2\varepsilon_c \le \lambda_{\min}(\boldsymbol{\Sigma}_w^\star)/2$ by Assumption~\ref{assump:label_uninform}(iii), we have
$\lambda_{\min}(\mathbf{A})-\|\bar{\mathbf{E}}\|_{\mathrm{op}}\ge \lambda_{\min}(\boldsymbol{\Sigma}_w^\star)/2$, and therefore
\[
\big\|\boldsymbol{\Sigma}_w^{-1}-\mathbf{A}^{-1}\big\|_{\mathrm{op}}
\le
\frac{2\varepsilon_c}{\lambda_{\min}(\boldsymbol{\Sigma}_w^\star)\cdot\lambda_{\min}(\boldsymbol{\Sigma}_w^\star)/2}
=
\frac{4\varepsilon_c}{\lambda_{\min}(\boldsymbol{\Sigma}_w^\star)^{2}}.
\]
Combining the two bounds,
\begin{align*}
\mathcal{J}(\bar{\mathbf{C}})
&=
(\boldsymbol{\mu}_1-\boldsymbol{\mu}_0)^\top\boldsymbol{\Sigma}_w^{-1}(\boldsymbol{\mu}_1-\boldsymbol{\mu}_0)\\
&\le
(\boldsymbol{\mu}_1-\boldsymbol{\mu}_0)^\top\mathbf{A}^{-1}(\boldsymbol{\mu}_1-\boldsymbol{\mu}_0)
+
\big\|\boldsymbol{\Sigma}_w^{-1}-\mathbf{A}^{-1}\big\|_{\mathrm{op}}\,
\|\boldsymbol{\mu}_1-\boldsymbol{\mu}_0\|_2^{2}\\
&\le
(\boldsymbol{\mu}_1-\boldsymbol{\mu}_0)^\top(\boldsymbol{\Sigma}_w^\star)^{-1}(\boldsymbol{\mu}_1-\boldsymbol{\mu}_0)
+
\frac{4\varepsilon_c\,(\|\boldsymbol{\mu}_1^\star-\boldsymbol{\mu}_0^\star\|_2+2\varepsilon_\mu)^{2}}{\lambda_{\min}(\boldsymbol{\Sigma}_w^\star)^{2}}.
\end{align*}
For the first term, writing $\boldsymbol{\mu}_1-\boldsymbol{\mu}_0 = (\boldsymbol{\mu}_1^\star-\boldsymbol{\mu}_0^\star)+(\boldsymbol{\eta}_1-\boldsymbol{\eta}_0)$ with $\|\boldsymbol{\eta}_1-\boldsymbol{\eta}_0\|_2\le 2\varepsilon_\mu$, expanding the quadratic form, and using $\|(\boldsymbol{\Sigma}_w^\star)^{-1}\|_{\mathrm{op}}\le 1/\lambda_{\min}(\boldsymbol{\Sigma}_w^\star)$,
\[
(\boldsymbol{\mu}_1-\boldsymbol{\mu}_0)^\top(\boldsymbol{\Sigma}_w^\star)^{-1}(\boldsymbol{\mu}_1-\boldsymbol{\mu}_0)
\le
\mathcal{J}(\bar{\mathbf{C}}^{\mathrm{info}})
+
\frac{4\varepsilon_\mu\,\|\boldsymbol{\mu}_1^\star-\boldsymbol{\mu}_0^\star\|_2 + 4\varepsilon_\mu^{2}}{\lambda_{\min}(\boldsymbol{\Sigma}_w^\star)}.
\]
Summing the two contributions recovers Eq.~\ref{eq:noise_hurts_approx} with $\Delta(\varepsilon_\mu,\varepsilon_c)$ as in Eq.~\ref{eq:slack_def}.
\end{proof}

Proposition~\ref{prop:noise_hurts} shows that, up to a slack vanishing with $(\varepsilon_\mu,\varepsilon_c)$, noisy reasoning steps do not improve Fisher discriminability and typically degrade it by inflating within-class variation.
This motivates recovering a trace close to the oracle informative trace before performing hallucination detection.

\subsection{\texorpdfstring{\(\kappa\)}{kappa}-NN distance separation from local density in the projected space}

We next analyze the practical score in Eq.~\ref{eq:theory_knn_score}.
The key quantity is the local neighbor count in Eq.~\ref{eq:theory_neighbor_count}: informative steps should have many nearby neighbors within a small radius, whereas noisy steps should have few nearby neighbors even within a larger radius.

\begin{assumption}[Local density separation]
\label{assump:density}
Conditioned on a center step $\bar{\mathbf{c}}_i$ and its type, the random indicators
\[
\mathbf{1}\!\left\{
D(\bar{\mathbf{z}}_j,\bar{\mathbf{z}}_i) \le r
\right\},
\qquad j\neq i,
\]
are independent Bernoulli variables for each fixed radius $r$.
Moreover, there exist radii $0<r_{\mathcal I}<r_{\mathcal N}$ such that
\begin{equation}
\mu_{\mathcal I}
:=
\mathbb{E}\!\left[N_i(r_{\mathcal I}) \mid \bar{\mathbf{c}}_i\in \bar{\mathbf{C}}^{\mathrm{info}}\right]
\ge 2\kappa,
\label{eq:mu_info_knn}
\end{equation}
and
\begin{equation}
\mu_{\mathcal N}
:=
\mathbb{E}\!\left[N_i(r_{\mathcal N}) \mid \bar{\mathbf{c}}_i\in \bar{\mathbf{C}}^{\mathrm{noise}}\right]
\le \kappa/2.
\label{eq:mu_noise_knn}
\end{equation}
\end{assumption}
Here, the non-bold $\mu_{\mathcal I}, \mu_{\mathcal N}$ denote scalar expectations of the local neighbor count, and should not be confused with the class-conditional trace means $\boldsymbol{\mu}_y$ used earlier.
This assumption formalizes the geometry induced by \model: informative steps form a compact local neighborhood, while noisy steps remain sparse and isolated.

\begin{proposition}[$\kappa$-NN distances separate informative and noisy steps]
\label{prop:knn_separation}
Under Assumption~\ref{assump:density}, the practical $\kappa$-NN score in Eq.~\ref{eq:theory_knn_score} satisfies
\begin{equation}
\mathbb{P}\!\left(
S_i > r_{\mathcal I}
\;\middle|\;
\bar{\mathbf{c}}_i\in \bar{\mathbf{C}}^{\mathrm{info}}
\right)
\le
\exp(-\kappa/4),
\label{eq:knn_info_tail}
\end{equation}
and
\begin{equation}
\mathbb{P}\!\left(
S_i \le r_{\mathcal N}
\;\middle|\;
\bar{\mathbf{c}}_i\in \bar{\mathbf{C}}^{\mathrm{noise}}
\right)
\le
\exp(-\kappa/6).
\label{eq:knn_noise_tail}
\end{equation}
\end{proposition}

\begin{proof}
By Eq.~\ref{eq:knn_count_equiv},
\[
S_i > r_{\mathcal I}
\iff
N_i(r_{\mathcal I}) < \kappa.
\]
For informative steps, Assumption~\ref{assump:density} gives
\[
\mu_{\mathcal I}
=
\mathbb{E}[N_i(r_{\mathcal I}) \mid \bar{\mathbf{c}}_i\in \bar{\mathbf{C}}^{\mathrm{info}}]
\ge 2\kappa.
\]
Therefore,
\[
\{N_i(r_{\mathcal I}) < \kappa\}
\subseteq
\left\{
N_i(r_{\mathcal I}) < \mu_{\mathcal I}/2
\right\}.
\]
Applying the multiplicative Chernoff bound to the lower tail,
\[
\mathbb{P}\!\left(
N_i(r_{\mathcal I}) < \mu_{\mathcal I}/2
\;\middle|\;
\bar{\mathbf{c}}_i\in \bar{\mathbf{C}}^{\mathrm{info}}
\right)
\le
\exp(-\mu_{\mathcal I}/8)
\le
\exp(-\kappa/4),
\]
which proves Eq.~\ref{eq:knn_info_tail}.

Similarly, by Eq.~\ref{eq:knn_count_equiv},
\[
S_i \le r_{\mathcal N}
\iff
N_i(r_{\mathcal N}) \ge \kappa.
\]
For noisy steps, Assumption~\ref{assump:density} gives
\[
\mu_{\mathcal N}
=
\mathbb{E}[N_i(r_{\mathcal N}) \mid \bar{\mathbf{c}}_i\in \bar{\mathbf{C}}^{\mathrm{noise}}]
\le \kappa/2.
\]
The event $\{N_i(r_{\mathcal N}) \ge \kappa\}$ is an upper-tail event at least as extreme as doubling the mean.
Applying the multiplicative Chernoff bound to the upper tail yields
\[
\mathbb{P}\!\left(
N_i(r_{\mathcal N}) \ge \kappa
\;\middle|\;
\bar{\mathbf{c}}_i\in \bar{\mathbf{C}}^{\mathrm{noise}}
\right)
\le
\exp(-\kappa/6),
\]
which proves Eq.~\ref{eq:knn_noise_tail}.
\end{proof}

Proposition~\ref{prop:knn_separation} shows that once the learned projection creates a sufficiently dense informative region and a sufficiently sparse noisy region, informative and noisy steps become separable directly under the practical $\kappa$-NN distance used at inference time.

\subsection{From \texorpdfstring{\(\kappa\)}{kappa}-NN score separation to rank-based filtering error}

We now analyze the practical rank-based rule that removes the top-$\zeta\%$ largest $\kappa$-NN distances within a trace.

Let
\begin{equation}
\bar{K}_{\mathcal N} := |\bar{\mathbf{C}}^{\mathrm{noise}}|
\label{eq:noise_count}
\end{equation}
denote the number of noisy steps in the trace, and define the score-separation event
\begin{equation}
\mathcal{E}_{\mathrm{sep}}
:=
\left\{
\max_{\bar{\mathbf{c}}_i\in \bar{\mathbf{C}}^{\mathrm{info}}} S_i \le r_{\mathcal I},
\quad
\min_{\bar{\mathbf{c}}_i\in \bar{\mathbf{C}}^{\mathrm{noise}}} S_i \ge r_{\mathcal N}
\right\}.
\label{eq:sep_event}
\end{equation}
Since $r_{\mathcal I}<r_{\mathcal N}$, on $\mathcal{E}_{\mathrm{sep}}$ every noisy step has a larger score than every informative step.

\begin{proposition}[Rank-based filtering error under practical $\kappa$-NN scores]
\label{prop:knn_ranking}
On the event $\mathcal{E}_{\mathrm{sep}}$, removing the top-$M$ steps ranked by $\{S_i\}_{i=1}^{\bar{K}}$ yields
\begin{equation}
\left|
\bar{\mathbf{C}}^{\mathrm{fil}}\cap \bar{\mathbf{C}}^{\mathrm{noise}}
\right|
=
(\bar{K}_{\mathcal N}-M)_+,
\qquad
\left|
\bar{\mathbf{C}}^{\mathrm{info}}\setminus \bar{\mathbf{C}}^{\mathrm{fil}}
\right|
=
(M-\bar{K}_{\mathcal N})_+,
\label{eq:rank_exact_counts}
\end{equation}
where $(x)_+ := \max\{x,0\}$.
Consequently,
\begin{align}
\alpha
&\le
\mathbb{E}\!\left[\frac{(\bar{K}_{\mathcal N}-M)_+}{\bar{K}}\right]
+
\mathbb{P}(\mathcal{E}_{\mathrm{sep}}^c),
\label{eq:alpha_rank_bound}
\\
\beta
&\le
\mathbb{E}\!\left[\frac{(M-\bar{K}_{\mathcal N})_+}{\bar{K}}\right]
+
\mathbb{P}(\mathcal{E}_{\mathrm{sep}}^c).
\label{eq:beta_rank_bound}
\end{align}
\end{proposition}

\begin{proof}
On $\mathcal{E}_{\mathrm{sep}}$, all noisy steps are ranked ahead of all informative steps by the practical $\kappa$-NN scores.
If $M \le \bar{K}_{\mathcal N}$, then the top-$M$ removed steps are all noisy.
Hence exactly $M$ noisy steps are removed, no informative steps are removed, and the number of retained noisy steps is $\bar{K}_{\mathcal N}-M$.
If $M > \bar{K}_{\mathcal N}$, then all noisy steps are removed and the remaining $M-\bar{K}_{\mathcal N}$ removed steps must be informative.
This proves Eq.~\ref{eq:rank_exact_counts}.

Now decompose
\[
\frac{|\bar{\mathbf{C}}^{\mathrm{fil}}\cap \bar{\mathbf{C}}^{\mathrm{noise}}|}{\bar{K}}
=
\frac{|\bar{\mathbf{C}}^{\mathrm{fil}}\cap \bar{\mathbf{C}}^{\mathrm{noise}}|}{\bar{K}}\mathbf{1}\{\mathcal{E}_{\mathrm{sep}}\}
+
\frac{|\bar{\mathbf{C}}^{\mathrm{fil}}\cap \bar{\mathbf{C}}^{\mathrm{noise}}|}{\bar{K}}\mathbf{1}\{\mathcal{E}_{\mathrm{sep}}^c\}.
\]
On $\mathcal{E}_{\mathrm{sep}}$, Eq.~\ref{eq:rank_exact_counts} gives
\[
\frac{|\bar{\mathbf{C}}^{\mathrm{fil}}\cap \bar{\mathbf{C}}^{\mathrm{noise}}|}{\bar{K}}
=
\frac{(\bar{K}_{\mathcal N}-M)_+}{\bar{K}}.
\]
On $\mathcal{E}_{\mathrm{sep}}^c$, the normalized count is at most $1$.
Taking expectations yields Eq.~\ref{eq:alpha_rank_bound}.
The proof of Eq.~\ref{eq:beta_rank_bound} is identical using the second identity in Eq.~\ref{eq:rank_exact_counts}.
\end{proof}

The bounds in Eqs.~\ref{eq:alpha_rank_bound}--\ref{eq:beta_rank_bound} separate two sources of error:
(1) \emph{count mismatch}, because the practical rule removes exactly $M=\lceil \zeta \bar{K}\rceil$ steps while the true number of noisy steps is $\bar{K}_{\mathcal N}$; and
(2) \emph{score overlap}, captured by $\mathbb{P}(\mathcal{E}_{\mathrm{sep}}^c)$.

\begin{corollary}[Filtering error under local density separation]
\label{cor:knn_filter}
Under Assumption~\ref{assump:density},
\begin{equation}
\mathbb{P}(\mathcal{E}_{\mathrm{sep}}^c)
\le
\mathbb{E}[\bar{K}]\,\big(\exp(-\kappa/4)+\exp(-\kappa/6)\big).
\label{eq:sep_event_prob}
\end{equation}
Consequently,
\begin{align}
\alpha
&\le
\mathbb{E}\!\left[\frac{(\bar{K}_{\mathcal N}-M)_+}{\bar{K}}\right]
+
\mathbb{E}[\bar{K}]\,\big(\exp(-\kappa/4)+\exp(-\kappa/6)\big),
\label{eq:alpha_knn_final}
\\
\beta
&\le
\mathbb{E}\!\left[\frac{(M-\bar{K}_{\mathcal N})_+}{\bar{K}}\right]
+
\mathbb{E}[\bar{K}]\,\big(\exp(-\kappa/4)+\exp(-\kappa/6)\big).
\label{eq:beta_knn_final}
\end{align}
\end{corollary}

\begin{proof}
Conditioned on a trace of length $\bar{K}$, a union bound and Proposition~\ref{prop:knn_separation} together with the trivial bounds $|\bar{\mathbf{C}}^{\mathrm{info}}|\le \bar{K}$ and $|\bar{\mathbf{C}}^{\mathrm{noise}}|\le \bar{K}$ give
\[
\begin{aligned}
\mathbb{P}(\mathcal{E}_{\mathrm{sep}}^c \mid \bar{K})
&\le
\sum_{\bar{\mathbf{c}}_i\in \bar{\mathbf{C}}^{\mathrm{info}}}
\mathbb{P}\!\left(S_i>r_{\mathcal I}\mid \bar{\mathbf{c}}_i\in \bar{\mathbf{C}}^{\mathrm{info}}\right)
\\
&\quad+
\sum_{\bar{\mathbf{c}}_i\in \bar{\mathbf{C}}^{\mathrm{noise}}}
\mathbb{P}\!\left(S_i\le r_{\mathcal N}\mid \bar{\mathbf{c}}_i\in \bar{\mathbf{C}}^{\mathrm{noise}}\right)
\\
&\le
\bar{K}\,\big(\exp(-\kappa/4)+\exp(-\kappa/6)\big).
\end{aligned}
\]
Taking expectation over $\bar{K}$ on both sides yields Eq.~\ref{eq:sep_event_prob}.
Substituting Eq.~\ref{eq:sep_event_prob} into Proposition~\ref{prop:knn_ranking} yields Eqs.~\ref{eq:alpha_knn_final}--\ref{eq:beta_knn_final}.
\end{proof}

\subsection{Representation gap induced by imperfect filtering}

We now bound the representation distortion caused by imperfect filtering, where some noisy steps may be retained and some informative steps may be lost.

\begin{assumption}[Bounded step representations]
\label{assump:bounded}
There exist constants $B_{\mathrm{info}}, B_{\mathrm{noise}} > 0$ such that
\[
\|\boldsymbol{\psi}(\bar{\mathbf{c}})\|_2 \le B_{\mathrm{info}}
\quad \text{for all } \bar{\mathbf{c}}\in \bar{\mathbf{C}}^{\mathrm{info}},
\qquad
\|\boldsymbol{\psi}(\bar{\mathbf{c}})\|_2 \le B_{\mathrm{noise}}
\quad \text{for all } \bar{\mathbf{c}}\in \bar{\mathbf{C}}^{\mathrm{noise}}.
\]
\end{assumption}

\begin{lemma}[Bounded representation gap]
\label{lem:repr_gap}
Let
\[
\boldsymbol{\Phi}^{\mathrm{fil}} := \boldsymbol{\Phi}(\bar{\mathbf{p}},\bar{\mathbf{C}}^{\mathrm{fil}},\bar{\mathbf{a}}),
\qquad
\boldsymbol{\Phi}^{\star} := \boldsymbol{\Phi}(\bar{\mathbf{p}},\bar{\mathbf{C}}^{\mathrm{info}},\bar{\mathbf{a}}).
\]
Under Assumption~\ref{assump:bounded}, the expected representation gap between the filtered and oracle traces satisfies
\begin{equation}
\mathbb{E}\big[\|\boldsymbol{\Phi}^{\mathrm{fil}} - \boldsymbol{\Phi}^{\star}\|_2\big]
\le
B_{\mathrm{noise}}\,\alpha + B_{\mathrm{info}}\,\beta.
\label{eq:repr_gap}
\end{equation}
\end{lemma}

\begin{proof}
Using the additive representation in Eq.~\ref{eq:theory_trace_repr}, the difference decomposes as
\begin{equation}
\boldsymbol{\Phi}^{\mathrm{fil}} - \boldsymbol{\Phi}^{\star}
=
\frac{1}{\bar{K}}
\left(
\sum_{\bar{\mathbf{c}}\in \bar{\mathbf{C}}^{\mathrm{fil}}\cap \bar{\mathbf{C}}^{\mathrm{noise}}} \boldsymbol{\psi}(\bar{\mathbf{c}})
-
\sum_{\bar{\mathbf{c}}\in \bar{\mathbf{C}}^{\mathrm{info}}\setminus \bar{\mathbf{C}}^{\mathrm{fil}}} \boldsymbol{\psi}(\bar{\mathbf{c}})
\right).
\label{eq:repr_gap_decomp}
\end{equation}
Taking norms and applying the triangle inequality and Assumption~\ref{assump:bounded},
\begin{equation}
\|\boldsymbol{\Phi}^{\mathrm{fil}} - \boldsymbol{\Phi}^{\star}\|_2
\le
\frac{B_{\mathrm{noise}}}{\bar{K}}\big|\bar{\mathbf{C}}^{\mathrm{fil}} \cap \bar{\mathbf{C}}^{\mathrm{noise}}\big|
+
\frac{B_{\mathrm{info}}}{\bar{K}}\big|\bar{\mathbf{C}}^{\mathrm{info}} \setminus \bar{\mathbf{C}}^{\mathrm{fil}}\big|.
\end{equation}
Taking expectation on both sides and using the definitions of $\alpha$ and $\beta$ in Eqs.~\ref{eq:theory_alpha}--\ref{eq:theory_beta} yields Eq.~\ref{eq:repr_gap}.
\end{proof}

\subsection{Detection risk on filtered traces}

Following common practice in representation analysis~\citep{alain2016understanding,azaria2023internal,zhang2026harnessing}, we analyze the detection risk under a linear probe on the trace representation. We now show that a small representation gap implies that the filtered trace achieves hallucination detection risk close to that of the oracle informative trace.

\begin{assumption}[Linear probe with Lipschitz loss]
\label{assump:linear}
The downstream detector is a linear probe, i.e., $g(\boldsymbol{\Phi}) = \mathbf{w}^\top \boldsymbol{\Phi}+b$ for some $\mathbf{w}\in\mathbb{R}^{d_{\Phi}}$ and $b\in\mathbb{R}$, and the loss $\ell(\cdot,y)$ is $L_0$-Lipschitz in its first argument.
Consequently, the composed loss $\ell(g(\cdot),y)$ is Lipschitz with respect to the trace representation with constant $L := L_0\|\mathbf{w}\|_2$:
\[
|\ell(g(\mathbf{u}),y)-\ell(g(\mathbf{v}),y)| \le L\|\mathbf{u}-\mathbf{v}\|_2
\qquad
\text{for all } \mathbf{u},\mathbf{v}\in\mathbb{R}^{d_{\Phi}}.
\]
\end{assumption}

\begin{theorem}[Risk bound on filtered traces]
\label{thm:risk_bound}
Under Assumptions~\ref{assump:bounded} and \ref{assump:linear}, for any linear probe $g$ satisfying Assumption~\ref{assump:linear},
\begin{equation}
\mathcal{R}(g;\bar{\mathbf{C}}^{\mathrm{fil}})
\le
\mathcal{R}(g;\bar{\mathbf{C}}^{\mathrm{info}})
+
L\big(B_{\mathrm{noise}}\,\alpha + B_{\mathrm{info}}\,\beta\big).
\label{eq:risk_bound}
\end{equation}
Moreover, under Assumptions~\ref{assump:density}, \ref{assump:bounded} and \ref{assump:linear},
\begin{align}
\mathcal{R}(g;\bar{\mathbf{C}}^{\mathrm{fil}})
\le\;&
\mathcal{R}(g;\bar{\mathbf{C}}^{\mathrm{info}})
+
L B_{\mathrm{noise}}
\left(
\mathbb{E}\!\left[\frac{(\bar{K}_{\mathcal N}-M)_+}{\bar{K}}\right]
+
\mathbb{E}[\bar{K}]\,\big(e^{-\kappa/4}+e^{-\kappa/6}\big)
\right)
\nonumber\\
&+
L B_{\mathrm{info}}
\left(
\mathbb{E}\!\left[\frac{(M-\bar{K}_{\mathcal N})_+}{\bar{K}}\right]
+
\mathbb{E}[\bar{K}]\,\big(e^{-\kappa/4}+e^{-\kappa/6}\big)
\right).
\label{eq:risk_bound_knn}
\end{align}
\end{theorem}

\begin{proof}
By the definition of population risk,
\begin{align}
\mathcal{R}(g;\bar{\mathbf{C}}^{\mathrm{fil}}) - \mathcal{R}(g;\bar{\mathbf{C}}^{\mathrm{info}})
&=
\mathbb{E}\big[
\ell(g(\boldsymbol{\Phi}^{\mathrm{fil}}),y)
-
\ell(g(\boldsymbol{\Phi}^{\star}),y)
\big]
\nonumber\\
&\le
\mathbb{E}\big[
\big|
\ell(g(\boldsymbol{\Phi}^{\mathrm{fil}}),y)
-
\ell(g(\boldsymbol{\Phi}^{\star}),y)
\big|
\big]
\nonumber\\
&\le
L\,\mathbb{E}\big[\|\boldsymbol{\Phi}^{\mathrm{fil}}-\boldsymbol{\Phi}^{\star}\|_2\big],
\end{align}
where the last step applies Assumption~\ref{assump:linear}.
Applying Lemma~\ref{lem:repr_gap} gives Eq.~\ref{eq:risk_bound}.
If Assumption~\ref{assump:density} additionally holds, substituting the bounds from Corollary~\ref{cor:knn_filter} for $\alpha$ and $\beta$ yields Eq.~\ref{eq:risk_bound_knn}.
\end{proof}

\begin{remark}[Role of $(\varepsilon_\mu,\varepsilon_c)$ in the overall risk bound]
\label{remark:risk_slack}
Theorem~\ref{thm:risk_bound} bounds the excess risk of the filtered trace relative to the \emph{oracle informative} risk $\mathcal{R}(g;\bar{\mathbf{C}}^{\mathrm{info}})$, and does not itself depend on $(\varepsilon_\mu,\varepsilon_c)$.
The slack parameters enter upstream through Proposition~\ref{prop:noise_hurts}, which controls the gap between $\mathcal{R}(g;\bar{\mathbf{C}})$ (using the full, unfiltered trace) and $\mathcal{R}(g;\bar{\mathbf{C}}^{\mathrm{info}})$ via $\Delta(\varepsilon_\mu,\varepsilon_c)$.
Combining the two yields, for a problem-dependent constant $C>0$ relating Fisher discriminability to Bayes-optimal linear-probe risk,
\[
\mathcal{R}(g;\bar{\mathbf{C}}^{\mathrm{fil}})
\le
\mathcal{R}(g;\bar{\mathbf{C}})
+
L\big(B_{\mathrm{noise}}\alpha+B_{\mathrm{info}}\beta\big)
+
C\,\Delta(\varepsilon_\mu,\varepsilon_c).
\]
This makes the role of the slack parameters explicit: the smaller $(\varepsilon_\mu,\varepsilon_c)$ and the smaller the filtering errors $(\alpha,\beta)$, the larger the improvement of the filtered trace over the unfiltered one.
\end{remark}

\begin{corollary}[Consistency under accurate practical filtering]
\label{cor:consistency}
Under Assumptions~\ref{assump:bounded} and \ref{assump:linear}, if
\[
\alpha \to 0
\qquad\text{and}\qquad
\beta \to 0,
\]
then
\[
\mathcal{R}(g;\bar{\mathbf{C}}^{\mathrm{fil}})
\to
\mathcal{R}(g;\bar{\mathbf{C}}^{\mathrm{info}}).
\]
Moreover, under Assumptions~\ref{assump:density}, \ref{assump:bounded} and \ref{assump:linear}, if
\[
\mathbb{E}\!\left[\frac{(\bar{K}_{\mathcal N}-M)_+}{\bar{K}}\right]\to 0,
\qquad
\mathbb{E}\!\left[\frac{(M-\bar{K}_{\mathcal N})_+}{\bar{K}}\right]\to 0,
\]
and
\[
\mathbb{E}[\bar{K}]\,\big(e^{-\kappa/4} + e^{-\kappa/6}\big) \to 0,
\]
then
\[
\mathcal{R}(g;\bar{\mathbf{C}}^{\mathrm{fil}})
\to
\mathcal{R}(g;\bar{\mathbf{C}}^{\mathrm{info}}).
\]
\end{corollary}

\begin{proof}
The first claim follows immediately from Theorem~\ref{thm:risk_bound}.
For the second claim, Corollary~\ref{cor:knn_filter} implies $\alpha\to 0$ and $\beta\to 0$ under the stated conditions.
Applying the first claim completes the proof.
\end{proof}

Theorem~\ref{thm:risk_bound} shows that the excess hallucination detection risk incurred by using the filtered trace is directly controlled by the quality of practical $\kappa$-NN-based filtering.
Combined with Proposition~\ref{prop:noise_hurts}, Proposition~\ref{prop:knn_separation}, and Proposition~\ref{prop:knn_ranking}, it gives a complete chain:
the learned projection creates informative regions that are locally dense and noisy regions that are locally sparse;
this local density gap separates the practical $\kappa$-NN scores;
under rank-based removal of the largest scores, the resulting filtering errors are small;
and small filtering errors imply that the filtered trace achieves risk close to that of the oracle informative trace.

\subsection{Connection to the shaping objective}

The formal results above explain why filtering quality matters.
We now discuss how the shaping objective in Eq.~\ref{eq:total_loss} of the main paper supports the practical $\kappa$-NN filtering rule.

\begin{remark}[Effect of the shaping loss on practical $\kappa$-NN filtering]
\label{remark:geometry}
The training objective in Eq.~\ref{eq:total_loss} is defined on the attention-based proxy sets $\mathcal{T}$ and $\mathcal{B}$.
For interpretation, we consider the population-level approximation that these proxies are reasonably aligned with the latent informative and noisy step sets $\bar{\mathbf{C}}^{\mathrm{info}}$ and $\bar{\mathbf{C}}^{\mathrm{noise}}$.

Under this approximation, the compactness term $\mathcal{L}_{\mathrm{compact}}$ pulls projected embeddings of proxy-informative steps together, which increases the local neighbor count $N_i(r)$ for informative steps at small radii and therefore tends to decrease their practical $\kappa$-NN distances.
The dispersion term $\mathcal{L}_{\mathrm{disperse}}$ discourages projected embeddings of proxy-noisy steps from forming dense neighborhoods, which decreases their local neighbor count and therefore tends to increase their practical $\kappa$-NN distances.
The separation term $\mathcal{L}_{\mathrm{separate}}$ pushes the two groups apart, reducing cross-group neighborhood overlap and making the score-separation event $\mathcal{E}_{\mathrm{sep}}$ more likely.

Together, these terms shape a projected embedding space in which informative steps occupy compact high-density neighborhoods while noisy steps remain sparse and isolated.
Proposition~\ref{prop:knn_separation} then implies improved separation under the practical $\kappa$-NN score, and Proposition~\ref{prop:knn_ranking} together with Theorem~\ref{thm:risk_bound} implies lower filtering error and tighter hallucination detection risk.
The ablation study in Section~\ref{sec:ablation} provides empirical support for this interpretation.
\end{remark}

\end{document}